\documentclass[runningheads,orivec]{llncs}

\usepackage[T1]{fontenc}
\usepackage[utf8]{inputenc}
\DeclareUnicodeCharacter{0130}{\.{I}} % Fixes capital İ
\usepackage{graphicx}
\usepackage{amsmath}
\usepackage{subcaption}  % For subfigure environment

\usepackage[T1]{fontenc}
\usepackage[utf8]{inputenc}
\usepackage{amsmath,amssymb,amsthm,bm,mathtools,physics}
\usepackage{placeins}
\usepackage{array}
\usepackage{authblk}
\usepackage{algorithm}
\usepackage{algpseudocode,setspace}
\usepackage[english]{babel}
\usepackage{url,lineno,microtype}
\usepackage{mdframed}
\usepackage{upgreek}
\usepackage[disable]{todonotes}
\usepackage{soul}
\usepackage{enumitem}
\usepackage{changepage}
\usepackage{wrapfig}
\usepackage{multirow}

\newmdtheoremenv{boxthm}{Theorem}
\newmdtheoremenv{boxlem}{Lemma}
\newmdtheoremenv{boxcor}{Corollary}
\newmdtheoremenv{boxdef}{Definition}
\usepackage{pgfplots}
\pgfplotsset{compat=1.17}
\usetikzlibrary{intersections}
\usepackage{bm}
\usetikzlibrary{positioning}
\usepgfplotslibrary{fillbetween}
\usepgfplotslibrary{groupplots}
\usetikzlibrary{shapes.geometric,backgrounds}
\usepackage[outline]{contour}

\usepackage{tikz}
\usepackage{xifthen}

\pgfdeclarelayer{background}    % declare background layer
\pgfsetlayers{background,main}  % set the order of the layers (main is the standard layer)
\definecolor{beige}{RGB}{245, 245, 220}

\definecolor{darkgrey}{RGB}{75, 75, 75}
\definecolor{lightgrey}{RGB}{250, 250, 250}

\usetikzlibrary{calc, arrows, fit, positioning, patterns,
      decorations.pathreplacing, shapes}
\tikzstyle{dash} = [dashed, -latex,>=latex]
\tikzstyle{line} = [draw, -latex,>=latex]
\tikzstyle{smallbox} = [draw, minimum size=5.0mm]
\tikzstyle{box} = [draw, minimum size=7.0mm]
\tikzstyle{bigbox} = [draw, minimum size=10.0mm]
\tikzstyle{switch} = [trapezium, trapezium angle=120, draw, rotate=90,
inner ysep=5pt, outer sep=5pt,
minimum height=7mm, minimum width=7mm]
\tikzstyle{roundbox} = [draw, circle, inner sep=0pt, minimum size=3mm]
\tikzstyle{clamped} = [draw, fill=black, minimum size=0.15cm]
\tikzstyle{msgcircle} = [shape=circle, draw, inner sep=0pt, minimum
size=4mm, fill=white, font=\scriptsize]
\tikzstyle{darkmsgcircle} = [shape=circle, draw, inner sep=0pt, minimum
size=4mm, fill=darkgrey, text=white, font=\scriptsize]
\tikzstyle{msgdoublecircle} = [shape=circle, double, double
distance=1.5pt, draw, inner sep=0pt, minimum size=5mm, fill=white]
\tikzstyle{darkmsgdoublecircle} = [shape=circle, double, double
distance=1.5pt, draw, inner sep=0pt, minimum size=5mm, fill=darkgrey,
text=white, font=\bfseries]
\newcommand{\msg}[6]{
      \ifthenelse{\isin{#1}{left} \AND \isin{#2}{down}}{
            \coordinate (anchor) at ($({#3})!{#5}!({#4})$);
            \node[xshift=-6.0mm] at (anchor) {#6};
            \node[xshift=-1.0mm] at (anchor) {$\downarrow$};
      }{}
      \ifthenelse{\isin{#1}{right} \AND \isin{#2}{down}}{
            \coordinate (anchor) at ($({#3})!{#5}!({#4})$);
            \node[xshift=6.0mm] at (anchor) {#6};
            \node[xshift=1.0mm] at (anchor) {$\downarrow$};
      }{}

      \ifthenelse{\isin{#1}{down} \AND \isin{#2}{right}}{
            \coordinate (anchor) at ($({#3})!{#5}!({#4})$);
            \node[ yshift=-4.0mm] at (anchor) {#6};
            \node[yshift=-1.0mm] at (anchor) {$\rightarrow$};
      }{}
      \ifthenelse{\isin{#1}{up} \AND \isin{#2}{right}}{
            \coordinate (anchor) at ($({#3})!{#5}!({#4})$);
            \node[ yshift=4.0mm] at (anchor) {#6};
            \node[yshift=1.0mm] at (anchor) {$\rightarrow$};
      }{}

      \ifthenelse{\isin{#1}{down} \AND \isin{#2}{left}}{
            \coordinate (anchor) at ($({#3})!{#5}!({#4})$);
            \node[ yshift=-4.0mm] at (anchor) {#6};
            \node[yshift=-1.0mm] at (anchor) {$\leftarrow$};
      }{}
      \ifthenelse{\isin{#1}{up} \AND \isin{#2}{left}}{
            \coordinate (anchor) at ($({#3})!{#5}!({#4})$);
            \node[ yshift=4.0mm] at (anchor) {#6};
            \node[yshift=1.0mm] at (anchor) {$\leftarrow$};
      }{}

      \ifthenelse{\isin{#1}{left} \AND \isin{#2}{up}}{
            \coordinate (anchor) at ($({#3})!{#5}!({#4})$);
            \node[ xshift=-6.0mm] at (anchor) {#6};
            \node[xshift=-1.0mm] at (anchor) {$\uparrow$};
      }{}
      \ifthenelse{\isin{#1}{right} \AND \isin{#2}{up}}{
            \coordinate (anchor) at ($({#3})!{#5}!({#4})$);
            \node[ xshift=6.0mm] at (anchor) {#6};
            \node[xshift=1.0mm] at (anchor) {$\uparrow$};
      }{}
}

\newcommand{\msgcircle}[6]{
      \ifthenelse{\isin{#1}{left} \AND \isin{#2}{down}}{
            \coordinate (anchor) at ($({#3})!{#5}!({#4})$);
            \node[msgcircle,xshift=-5.0mm] at (anchor) {#6};
            \node[xshift=-1.5mm] at (anchor) {$\downarrow$};
      }{}
      \ifthenelse{\isin{#1}{right} \AND \isin{#2}{down}}{
            \coordinate (anchor) at ($({#3})!{#5}!({#4})$);
            \node[msgcircle,xshift=5.0mm] at (anchor) {#6};
            \node[xshift=1.5mm] at (anchor) {$\downarrow$};
      }{}

      \ifthenelse{\isin{#1}{down} \AND \isin{#2}{right}}{
            \coordinate (anchor) at ($({#3})!{#5}!({#4})$);
            \node[msgcircle, yshift=-5.0mm] at (anchor) {#6};
            \node[yshift=-2.0mm] at (anchor) {$\rightarrow$};
      }{}
      \ifthenelse{\isin{#1}{up} \AND \isin{#2}{right}}{
            \coordinate (anchor) at ($({#3})!{#5}!({#4})$);
            \node[msgcircle, yshift=5.0mm] at (anchor) {#6};
            \node[yshift=2.0mm] at (anchor) {$\rightarrow$};
      }{}

      \ifthenelse{\isin{#1}{down} \AND \isin{#2}{left}}{
            \coordinate (anchor) at ($({#3})!{#5}!({#4})$);
            \node[msgcircle, yshift=-5.0mm] at (anchor) {#6};
            \node[yshift=-2.0mm] at (anchor) {$\leftarrow$};
      }{}
      \ifthenelse{\isin{#1}{up} \AND \isin{#2}{left}}{
            \coordinate (anchor) at ($({#3})!{#5}!({#4})$);
            \node[msgcircle, yshift=5.0mm] at (anchor) {#6};
            \node[yshift=2.0mm] at (anchor) {$\leftarrow$};
      }{}

      \ifthenelse{\isin{#1}{left} \AND \isin{#2}{up}}{
            \coordinate (anchor) at ($({#3})!{#5}!({#4})$);
            \node[msgcircle, xshift=-5.0mm] at (anchor) {#6};
            \node[xshift=-1.5mm] at (anchor) {$\uparrow$};
      }{}
      \ifthenelse{\isin{#1}{right} \AND \isin{#2}{up}}{
            \coordinate (anchor) at ($({#3})!{#5}!({#4})$);
            \node[msgcircle, xshift=5.0mm] at (anchor) {#6};
            \node[xshift=1.5mm] at (anchor) {$\uparrow$};
      }{}
}

\newcommand{\darkmsg}[6]{
      \ifthenelse{\isin{#1}{left} \AND \isin{#2}{down}}{
            \coordinate (anchor) at ($({#3})!{#5}!({#4})$);
            \node[darkmsgcircle, xshift=-5mm] at (anchor) {#6};
            \node[xshift=-1.5mm] at (anchor) {$\downarrow$};
      }{}
      \ifthenelse{\isin{#1}{right} \AND \isin{#2}{down}}{
            \coordinate (anchor) at ($({#3})!{#5}!({#4})$);
            \node[darkmsgcircle, xshift=5mm] at (anchor) {#6};
            \node[xshift=1.5mm] at (anchor) {$\downarrow$};
      }{}

      \ifthenelse{\isin{#1}{down} \AND \isin{#2}{right}}{
            \coordinate (anchor) at ($({#3})!{#5}!({#4})$);
            \node[darkmsgcircle, yshift=-5.0mm] at (anchor) {#6};
            \node[yshift=-2.0mm] at (anchor) {$\rightarrow$};
      }{}
      \ifthenelse{\isin{#1}{up} \AND \isin{#2}{right}}{
            \coordinate (anchor) at ($({#3})!{#5}!({#4})$);
            \node[darkmsgcircle, yshift=5.0mm] at (anchor) {#6};
            \node[yshift=2.0mm] at (anchor) {$\rightarrow$};
      }{}

      \ifthenelse{\isin{#1}{down} \AND \isin{#2}{left}}{
            \coordinate (anchor) at ($({#3})!{#5}!({#4})$);
            \node[darkmsgcircle, yshift=-5.0mm] at (anchor) {#6};
            \node[yshift=-2.0mm] at (anchor) {$\leftarrow$};
      }{}
      \ifthenelse{\isin{#1}{up} \AND \isin{#2}{left}}{
            \coordinate (anchor) at ($({#3})!{#5}!({#4})$);
            \node[darkmsgcircle, yshift=5.0mm] at (anchor) {#6};
            \node[yshift=2.0mm] at (anchor) {$\leftarrow$};
      }{}

      \ifthenelse{\isin{#1}{left} \AND \isin{#2}{up}}{
            \coordinate (anchor) at ($({#3})!{#5}!({#4})$);
            \node[darkmsgcircle, xshift=-5.0mm] at (anchor) {#6};
            \node[xshift=-1.5mm] at (anchor) {$\uparrow$};
      }{}
      \ifthenelse{\isin{#1}{right} \AND \isin{#2}{up}}{
            \coordinate (anchor) at ($({#3})!{#5}!({#4})$);
            \node[darkmsgcircle, xshift=5.0mm] at (anchor) {#6};
            \node[xshift=1.5mm] at (anchor) {$\uparrow$};
      }{}
}

\newcommand{\bwmsg}[6]{
      \ifthenelse{\isin{#1}{left} \AND \isin{#2}{down}}{
            \coordinate (anchor) at ($({#3})!{#5}!({#4})$);
            \node[msgdoublecircle, xshift=-5.5mm] at (anchor) {#6};
            \node[xshift=-1.5mm] at (anchor) {$\downarrow$};
      }{}
      \ifthenelse{\isin{#1}{right} \AND \isin{#2}{down}}{
            \coordinate (anchor) at ($({#3})!{#5}!({#4})$);
            \node[msgdoublecircle, xshift=5.5mm] at (anchor) {#6};
            \node[xshift=1.5mm] at (anchor) {$\downarrow$};
      }{}

      \ifthenelse{\isin{#1}{down} \AND \isin{#2}{right}}{
            \coordinate (anchor) at ($({#3})!{#5}!({#4})$);
            \node[msgdoublecircle, yshift=-6.0mm] at (anchor) {#6};
            \node[yshift=-2.0mm] at (anchor) {$\rightarrow$};
      }{}
      \ifthenelse{\isin{#1}{up} \AND \isin{#2}{right}}{
            \coordinate (anchor) at ($({#3})!{#5}!({#4})$);
            \node[msgdoublecircle, yshift=6.0mm] at (anchor) {#6};
            \node[yshift=2.0mm] at (anchor) {$\rightarrow$};
      }{}

      \ifthenelse{\isin{#1}{down} \AND \isin{#2}{left}}{
            \coordinate (anchor) at ($({#3})!{#5}!({#4})$);
            \node[msgdoublecircle, yshift=-6.0mm] at (anchor) {#6};
            \node[yshift=-2.0mm] at (anchor) {$\leftarrow$};
      }{}
      \ifthenelse{\isin{#1}{up} \AND \isin{#2}{left}}{
            \coordinate (anchor) at ($({#3})!{#5}!({#4})$);
            \node[msgdoublecircle, yshift=6.0mm] at (anchor) {#6};
            \node[yshift=2.0mm] at (anchor) {$\leftarrow$};
      }{}

      \ifthenelse{\isin{#1}{left} \AND \isin{#2}{up}}{
            \coordinate (anchor) at ($({#3})!{#5}!({#4})$);
            \node[msgdoublecircle, xshift=-5.5mm] at (anchor) {#6};
            \node[xshift=-1.5mm] at (anchor) {$\uparrow$};
      }{}
      \ifthenelse{\isin{#1}{right} \AND \isin{#2}{up}}{
            \coordinate (anchor) at ($({#3})!{#5}!({#4})$);
            \node[msgdoublecircle, xshift=5.5mm] at (anchor) {#6};
            \node[xshift=1.5mm] at (anchor) {$\uparrow$};
      }{}
}

\newcommand{\bwdarkmsg}[6]{
      \ifthenelse{\isin{#1}{left} \AND \isin{#2}{down}}{
            \coordinate (anchor) at ($({#3})!{#5}!({#4})$);
            \node[darkmsgdoublecircle, xshift=-5.5mm] at (anchor) {#6};
            \node[xshift=-1.5mm] at (anchor) {$\downarrow$};
      }{}
      \ifthenelse{\isin{#1}{right} \AND \isin{#2}{down}}{
            \coordinate (anchor) at ($({#3})!{#5}!({#4})$);
            \node[darkmsgdoublecircle, xshift=5.5mm] at (anchor) {#6};
            \node[xshift=1.5mm] at (anchor) {$\downarrow$};
      }{}

      \ifthenelse{\isin{#1}{down} \AND \isin{#2}{right}}{
            \coordinate (anchor) at ($({#3})!{#5}!({#4})$);
            \node[darkmsgdoublecircle, yshift=-6.0mm] at (anchor) {#6};
            \node[yshift=-2.0mm] at (anchor) {$\rightarrow$};
      }{}
      \ifthenelse{\isin{#1}{up} \AND \isin{#2}{right}}{
            \coordinate (anchor) at ($({#3})!{#5}!({#4})$);
            \node[darkmsgdoublecircle, yshift=6.0mm] at (anchor) {#6};
            \node[yshift=2.0mm] at (anchor) {$\rightarrow$};
      }{}

      \ifthenelse{\isin{#1}{down} \AND \isin{#2}{left}}{
            \coordinate (anchor) at ($({#3})!{#5}!({#4})$);
            \node[darkmsgdoublecircle, yshift=-6.0mm] at (anchor) {#6};
            \node[yshift=-2.0mm] at (anchor) {$\leftarrow$};
      }{}
      \ifthenelse{\isin{#1}{up} \AND \isin{#2}{left}}{
            \coordinate (anchor) at ($({#3})!{#5}!({#4})$);
            \node[darkmsgdoublecircle, yshift=6.0mm] at (anchor) {#6};
            \node[yshift=2.0mm] at (anchor) {$\leftarrow$};
      }{}

      \ifthenelse{\isin{#1}{left} \AND \isin{#2}{up}}{
            \coordinate (anchor) at ($({#3})!{#5}!({#4})$);
            \node[darkmsgdoublecircle, xshift=-5.5mm] at (anchor) {#6};
            \node[xshift=-1.5mm] at (anchor) {$\uparrow$};
      }{}
      \ifthenelse{\isin{#1}{right} \AND \isin{#2}{up}}{
            \coordinate (anchor) at ($({#3})!{#5}!({#4})$);
            \node[darkmsgdoublecircle, xshift=5.5mm] at (anchor) {#6};
            \node[xshift=1.5mm] at (anchor) {$\uparrow$};
      }{}
}

\makeatletter
\DeclareRobustCommand{\cev}[1]{%
      \mathpalette\do@cev{#1}%
}
\newcommand{\do@cev}[2]{%
      \fix@cev{#1}{+}%

      \reflectbox{$\m@th#1\vec{\reflectbox{$\fix@cev{#1}{-}\m@th#1#2\fix@cev{#1}{+}$}}$}%
      \fix@cev{#1}{-}%
}
\newcommand{\fix@cev}[2]{%
      \ifx#1\displaystyle
            \mkern#23mu
      \else
            \ifx#1\textstyle
                  \mkern#23mu
            \else
                  \ifx#1\scriptstyle
                        \mkern#22mu
                  \else
                        \mkern#22mu
                  \fi
            \fi
      \fi
}

\usetikzlibrary{arrows.meta}
\usetikzlibrary{backgrounds}
\usepgfplotslibrary{patchplots}
\usepgfplotslibrary{fillbetween}
\pgfplotsset{%
      layers/standard/.define layer set={%
                  background,axis background,axis grid,axis ticks,axis lines,axis tick labels,pre main,main,axis descriptions,axis foreground%
            }{
                  grid style={/pgfplots/on layer=axis grid},%
                  tick style={/pgfplots/on layer=axis ticks},%
                  axis line style={/pgfplots/on layer=axis lines},%
                  label style={/pgfplots/on layer=axis descriptions},%
                  legend style={/pgfplots/on layer=axis descriptions},%
                  title style={/pgfplots/on layer=axis descriptions},%
                  colorbar style={/pgfplots/on layer=axis descriptions},%
                  ticklabel style={/pgfplots/on layer=axis tick labels},%
                  axis background@ style={/pgfplots/on layer=axis background},%
                  3d box foreground style={/pgfplots/on layer=axis foreground},%
            },
}

\pgfplotsset{
      colormap={plots2}{rgb(0.00000000)=(0.26700400,0.00487400,0.32941500)
                  rgb(0.00392157)=(0.26851000,0.00960500,0.33542700)
                  rgb(0.00784314)=(0.26994400,0.01462500,0.34137900)
                  rgb(0.01176471)=(0.27130500,0.01994200,0.34726900)
                  rgb(0.01568627)=(0.27259400,0.02556300,0.35309300)
                  rgb(0.01960784)=(0.27380900,0.03149700,0.35885300)
                  rgb(0.02352941)=(0.27495200,0.03775200,0.36454300)
                  rgb(0.02745098)=(0.27602200,0.04416700,0.37016400)
                  rgb(0.03137255)=(0.27701800,0.05034400,0.37571500)
                  rgb(0.03529412)=(0.27794100,0.05632400,0.38119100)
                  rgb(0.03921569)=(0.27879100,0.06214500,0.38659200)
                  rgb(0.04313725)=(0.27956600,0.06783600,0.39191700)
                  rgb(0.04705882)=(0.28026700,0.07341700,0.39716300)
                  rgb(0.05098039)=(0.28089400,0.07890700,0.40232900)
                  rgb(0.05490196)=(0.28144600,0.08432000,0.40741400)
                  rgb(0.05882353)=(0.28192400,0.08966600,0.41241500)
                  rgb(0.06274510)=(0.28232700,0.09495500,0.41733100)
                  rgb(0.06666667)=(0.28265600,0.10019600,0.42216000)
                  rgb(0.07058824)=(0.28291000,0.10539300,0.42690200)
                  rgb(0.07450980)=(0.28309100,0.11055300,0.43155400)
                  rgb(0.07843137)=(0.28319700,0.11568000,0.43611500)
                  rgb(0.08235294)=(0.28322900,0.12077700,0.44058400)
                  rgb(0.08627451)=(0.28318700,0.12584800,0.44496000)
                  rgb(0.09019608)=(0.28307200,0.13089500,0.44924100)
                  rgb(0.09411765)=(0.28288400,0.13592000,0.45342700)
                  rgb(0.09803922)=(0.28262300,0.14092600,0.45751700)
                  rgb(0.10196078)=(0.28229000,0.14591200,0.46151000)
                  rgb(0.10588235)=(0.28188700,0.15088100,0.46540500)
                  rgb(0.10980392)=(0.28141200,0.15583400,0.46920100)
                  rgb(0.11372549)=(0.28086800,0.16077100,0.47289900)
                  rgb(0.11764706)=(0.28025500,0.16569300,0.47649800)
                  rgb(0.12156863)=(0.27957400,0.17059900,0.47999700)
                  rgb(0.12549020)=(0.27882600,0.17549000,0.48339700)
                  rgb(0.12941176)=(0.27801200,0.18036700,0.48669700)
                  rgb(0.13333333)=(0.27713400,0.18522800,0.48989800)
                  rgb(0.13725490)=(0.27619400,0.19007400,0.49300100)
                  rgb(0.14117647)=(0.27519100,0.19490500,0.49600500)
                  rgb(0.14509804)=(0.27412800,0.19972100,0.49891100)
                  rgb(0.14901961)=(0.27300600,0.20452000,0.50172100)
                  rgb(0.15294118)=(0.27182800,0.20930300,0.50443400)
                  rgb(0.15686275)=(0.27059500,0.21406900,0.50705200)
                  rgb(0.16078431)=(0.26930800,0.21881800,0.50957700)
                  rgb(0.16470588)=(0.26796800,0.22354900,0.51200800)
                  rgb(0.16862745)=(0.26658000,0.22826200,0.51434900)
                  rgb(0.17254902)=(0.26514500,0.23295600,0.51659900)
                  rgb(0.17647059)=(0.26366300,0.23763100,0.51876200)
                  rgb(0.18039216)=(0.26213800,0.24228600,0.52083700)
                  rgb(0.18431373)=(0.26057100,0.24692200,0.52282800)
                  rgb(0.18823529)=(0.25896500,0.25153700,0.52473600)
                  rgb(0.19215686)=(0.25732200,0.25613000,0.52656300)
                  rgb(0.19607843)=(0.25564500,0.26070300,0.52831200)
                  rgb(0.20000000)=(0.25393500,0.26525400,0.52998300)
                  rgb(0.20392157)=(0.25219400,0.26978300,0.53157900)
                  rgb(0.20784314)=(0.25042500,0.27429000,0.53310300)
                  rgb(0.21176471)=(0.24862900,0.27877500,0.53455600)
                  rgb(0.21568627)=(0.24681100,0.28323700,0.53594100)
                  rgb(0.21960784)=(0.24497200,0.28767500,0.53726000)
                  rgb(0.22352941)=(0.24311300,0.29209200,0.53851600)
                  rgb(0.22745098)=(0.24123700,0.29648500,0.53970900)
                  rgb(0.23137255)=(0.23934600,0.30085500,0.54084400)
                  rgb(0.23529412)=(0.23744100,0.30520200,0.54192100)
                  rgb(0.23921569)=(0.23552600,0.30952700,0.54294400)
                  rgb(0.24313725)=(0.23360300,0.31382800,0.54391400)
                  rgb(0.24705882)=(0.23167400,0.31810600,0.54483400)
                  rgb(0.25098039)=(0.22973900,0.32236100,0.54570600)
                  rgb(0.25490196)=(0.22780200,0.32659400,0.54653200)
                  rgb(0.25882353)=(0.22586300,0.33080500,0.54731400)
                  rgb(0.26274510)=(0.22392500,0.33499400,0.54805300)
                  rgb(0.26666667)=(0.22198900,0.33916100,0.54875200)
                  rgb(0.27058824)=(0.22005700,0.34330700,0.54941300)
                  rgb(0.27450980)=(0.21813000,0.34743200,0.55003800)
                  rgb(0.27843137)=(0.21621000,0.35153500,0.55062700)
                  rgb(0.28235294)=(0.21429800,0.35561900,0.55118400)
                  rgb(0.28627451)=(0.21239500,0.35968300,0.55171000)
                  rgb(0.29019608)=(0.21050300,0.36372700,0.55220600)
                  rgb(0.29411765)=(0.20862300,0.36775200,0.55267500)
                  rgb(0.29803922)=(0.20675600,0.37175800,0.55311700)
                  rgb(0.30196078)=(0.20490300,0.37574600,0.55353300)
                  rgb(0.30588235)=(0.20306300,0.37971600,0.55392500)
                  rgb(0.30980392)=(0.20123900,0.38367000,0.55429400)
                  rgb(0.31372549)=(0.19943000,0.38760700,0.55464200)
                  rgb(0.31764706)=(0.19763600,0.39152800,0.55496900)
                  rgb(0.32156863)=(0.19586000,0.39543300,0.55527600)
                  rgb(0.32549020)=(0.19410000,0.39932300,0.55556500)
                  rgb(0.32941176)=(0.19235700,0.40319900,0.55583600)
                  rgb(0.33333333)=(0.19063100,0.40706100,0.55608900)
                  rgb(0.33725490)=(0.18892300,0.41091000,0.55632600)
                  rgb(0.34117647)=(0.18723100,0.41474600,0.55654700)
                  rgb(0.34509804)=(0.18555600,0.41857000,0.55675300)
                  rgb(0.34901961)=(0.18389800,0.42238300,0.55694400)
                  rgb(0.35294118)=(0.18225600,0.42618400,0.55712000)
                  rgb(0.35686275)=(0.18062900,0.42997500,0.55728200)
                  rgb(0.36078431)=(0.17901900,0.43375600,0.55743000)
                  rgb(0.36470588)=(0.17742300,0.43752700,0.55756500)
                  rgb(0.36862745)=(0.17584100,0.44129000,0.55768500)
                  rgb(0.37254902)=(0.17427400,0.44504400,0.55779200)
                  rgb(0.37647059)=(0.17271900,0.44879100,0.55788500)
                  rgb(0.38039216)=(0.17117600,0.45253000,0.55796500)
                  rgb(0.38431373)=(0.16964600,0.45626200,0.55803000)
                  rgb(0.38823529)=(0.16812600,0.45998800,0.55808200)
                  rgb(0.39215686)=(0.16661700,0.46370800,0.55811900)
                  rgb(0.39607843)=(0.16511700,0.46742300,0.55814100)
                  rgb(0.40000000)=(0.16362500,0.47113300,0.55814800)
                  rgb(0.40392157)=(0.16214200,0.47483800,0.55814000)
                  rgb(0.40784314)=(0.16066500,0.47854000,0.55811500)
                  rgb(0.41176471)=(0.15919400,0.48223700,0.55807300)
                  rgb(0.41568627)=(0.15772900,0.48593200,0.55801300)
                  rgb(0.41960784)=(0.15627000,0.48962400,0.55793600)
                  rgb(0.42352941)=(0.15481500,0.49331300,0.55784000)
                  rgb(0.42745098)=(0.15336400,0.49700000,0.55772400)
                  rgb(0.43137255)=(0.15191800,0.50068500,0.55758700)
                  rgb(0.43529412)=(0.15047600,0.50436900,0.55743000)
                  rgb(0.43921569)=(0.14903900,0.50805100,0.55725000)
                  rgb(0.44313725)=(0.14760700,0.51173300,0.55704900)
                  rgb(0.44705882)=(0.14618000,0.51541300,0.55682300)
                  rgb(0.45098039)=(0.14475900,0.51909300,0.55657200)
                  rgb(0.45490196)=(0.14334300,0.52277300,0.55629500)
                  rgb(0.45882353)=(0.14193500,0.52645300,0.55599100)
                  rgb(0.46274510)=(0.14053600,0.53013200,0.55565900)
                  rgb(0.46666667)=(0.13914700,0.53381200,0.55529800)
                  rgb(0.47058824)=(0.13777000,0.53749200,0.55490600)
                  rgb(0.47450980)=(0.13640800,0.54117300,0.55448300)
                  rgb(0.47843137)=(0.13506600,0.54485300,0.55402900)
                  rgb(0.48235294)=(0.13374300,0.54853500,0.55354100)
                  rgb(0.48627451)=(0.13244400,0.55221600,0.55301800)
                  rgb(0.49019608)=(0.13117200,0.55589900,0.55245900)
                  rgb(0.49411765)=(0.12993300,0.55958200,0.55186400)
                  rgb(0.49803922)=(0.12872900,0.56326500,0.55122900)
                  rgb(0.50196078)=(0.12756800,0.56694900,0.55055600)
                  rgb(0.50588235)=(0.12645300,0.57063300,0.54984100)
                  rgb(0.50980392)=(0.12539400,0.57431800,0.54908600)
                  rgb(0.51372549)=(0.12439500,0.57800200,0.54828700)
                  rgb(0.51764706)=(0.12346300,0.58168700,0.54744500)
                  rgb(0.52156863)=(0.12260600,0.58537100,0.54655700)
                  rgb(0.52549020)=(0.12183100,0.58905500,0.54562300)
                  rgb(0.52941176)=(0.12114800,0.59273900,0.54464100)
                  rgb(0.53333333)=(0.12056500,0.59642200,0.54361100)
                  rgb(0.53725490)=(0.12009200,0.60010400,0.54253000)
                  rgb(0.54117647)=(0.11973800,0.60378500,0.54140000)
                  rgb(0.54509804)=(0.11951200,0.60746400,0.54021800)
                  rgb(0.54901961)=(0.11942300,0.61114100,0.53898200)
                  rgb(0.55294118)=(0.11948300,0.61481700,0.53769200)
                  rgb(0.55686275)=(0.11969900,0.61849000,0.53634700)
                  rgb(0.56078431)=(0.12008100,0.62216100,0.53494600)
                  rgb(0.56470588)=(0.12063800,0.62582800,0.53348800)
                  rgb(0.56862745)=(0.12138000,0.62949200,0.53197300)
                  rgb(0.57254902)=(0.12231200,0.63315300,0.53039800)
                  rgb(0.57647059)=(0.12344400,0.63680900,0.52876300)
                  rgb(0.58039216)=(0.12478000,0.64046100,0.52706800)
                  rgb(0.58431373)=(0.12632600,0.64410700,0.52531100)
                  rgb(0.58823529)=(0.12808700,0.64774900,0.52349100)
                  rgb(0.59215686)=(0.13006700,0.65138400,0.52160800)
                  rgb(0.59607843)=(0.13226800,0.65501400,0.51966100)
                  rgb(0.60000000)=(0.13469200,0.65863600,0.51764900)
                  rgb(0.60392157)=(0.13733900,0.66225200,0.51557100)
                  rgb(0.60784314)=(0.14021000,0.66585900,0.51342700)
                  rgb(0.61176471)=(0.14330300,0.66945900,0.51121500)
                  rgb(0.61568627)=(0.14661600,0.67305000,0.50893600)
                  rgb(0.61960784)=(0.15014800,0.67663100,0.50658900)
                  rgb(0.62352941)=(0.15389400,0.68020300,0.50417200)
                  rgb(0.62745098)=(0.15785100,0.68376500,0.50168600)
                  rgb(0.63137255)=(0.16201600,0.68731600,0.49912900)
                  rgb(0.63529412)=(0.16638300,0.69085600,0.49650200)
                  rgb(0.63921569)=(0.17094800,0.69438400,0.49380300)
                  rgb(0.64313725)=(0.17570700,0.69790000,0.49103300)
                  rgb(0.64705882)=(0.18065300,0.70140200,0.48818900)
                  rgb(0.65098039)=(0.18578300,0.70489100,0.48527300)
                  rgb(0.65490196)=(0.19109000,0.70836600,0.48228400)
                  rgb(0.65882353)=(0.19657100,0.71182700,0.47922100)
                  rgb(0.66274510)=(0.20221900,0.71527200,0.47608400)
                  rgb(0.66666667)=(0.20803000,0.71870100,0.47287300)
                  rgb(0.67058824)=(0.21400000,0.72211400,0.46958800)
                  rgb(0.67450980)=(0.22012400,0.72550900,0.46622600)
                  rgb(0.67843137)=(0.22639700,0.72888800,0.46278900)
                  rgb(0.68235294)=(0.23281500,0.73224700,0.45927700)
                  rgb(0.68627451)=(0.23937400,0.73558800,0.45568800)
                  rgb(0.69019608)=(0.24607000,0.73891000,0.45202400)
                  rgb(0.69411765)=(0.25289900,0.74221100,0.44828400)
                  rgb(0.69803922)=(0.25985700,0.74549200,0.44446700)
                  rgb(0.70196078)=(0.26694100,0.74875100,0.44057300)
                  rgb(0.70588235)=(0.27414900,0.75198800,0.43660100)
                  rgb(0.70980392)=(0.28147700,0.75520300,0.43255200)
                  rgb(0.71372549)=(0.28892100,0.75839400,0.42842600)
                  rgb(0.71764706)=(0.29647900,0.76156100,0.42422300)
                  rgb(0.72156863)=(0.30414800,0.76470400,0.41994300)
                  rgb(0.72549020)=(0.31192500,0.76782200,0.41558600)
                  rgb(0.72941176)=(0.31980900,0.77091400,0.41115200)
                  rgb(0.73333333)=(0.32779600,0.77398000,0.40664000)
                  rgb(0.73725490)=(0.33588500,0.77701800,0.40204900)
                  rgb(0.74117647)=(0.34407400,0.78002900,0.39738100)
                  rgb(0.74509804)=(0.35236000,0.78301100,0.39263600)
                  rgb(0.74901961)=(0.36074100,0.78596400,0.38781400)
                  rgb(0.75294118)=(0.36921400,0.78888800,0.38291400)
                  rgb(0.75686275)=(0.37777900,0.79178100,0.37793900)
                  rgb(0.76078431)=(0.38643300,0.79464400,0.37288600)
                  rgb(0.76470588)=(0.39517400,0.79747500,0.36775700)
                  rgb(0.76862745)=(0.40400100,0.80027500,0.36255200)
                  rgb(0.77254902)=(0.41291300,0.80304100,0.35726900)
                  rgb(0.77647059)=(0.42190800,0.80577400,0.35191000)
                  rgb(0.78039216)=(0.43098300,0.80847300,0.34647600)
                  rgb(0.78431373)=(0.44013700,0.81113800,0.34096700)
                  rgb(0.78823529)=(0.44936800,0.81376800,0.33538400)
                  rgb(0.79215686)=(0.45867400,0.81636300,0.32972700)
                  rgb(0.79607843)=(0.46805300,0.81892100,0.32399800)
                  rgb(0.80000000)=(0.47750400,0.82144400,0.31819500)
                  rgb(0.80392157)=(0.48702600,0.82392900,0.31232100)
                  rgb(0.80784314)=(0.49661500,0.82637600,0.30637700)
                  rgb(0.81176471)=(0.50627100,0.82878600,0.30036200)
                  rgb(0.81568627)=(0.51599200,0.83115800,0.29427900)
                  rgb(0.81960784)=(0.52577600,0.83349100,0.28812700)
                  rgb(0.82352941)=(0.53562100,0.83578500,0.28190800)
                  rgb(0.82745098)=(0.54552400,0.83803900,0.27562600)
                  rgb(0.83137255)=(0.55548400,0.84025400,0.26928100)
                  rgb(0.83529412)=(0.56549800,0.84243000,0.26287700)
                  rgb(0.83921569)=(0.57556300,0.84456600,0.25641500)
                  rgb(0.84313725)=(0.58567800,0.84666100,0.24989700)
                  rgb(0.84705882)=(0.59583900,0.84871700,0.24332900)
                  rgb(0.85098039)=(0.60604500,0.85073300,0.23671200)
                  rgb(0.85490196)=(0.61629300,0.85270900,0.23005200)
                  rgb(0.85882353)=(0.62657900,0.85464500,0.22335300)
                  rgb(0.86274510)=(0.63690200,0.85654200,0.21662000)
                  rgb(0.86666667)=(0.64725700,0.85840000,0.20986100)
                  rgb(0.87058824)=(0.65764200,0.86021900,0.20308200)
                  rgb(0.87450980)=(0.66805400,0.86199900,0.19629300)
                  rgb(0.87843137)=(0.67848900,0.86374200,0.18950300)
                  rgb(0.88235294)=(0.68894400,0.86544800,0.18272500)
                  rgb(0.88627451)=(0.69941500,0.86711700,0.17597100)
                  rgb(0.89019608)=(0.70989800,0.86875100,0.16925700)
                  rgb(0.89411765)=(0.72039100,0.87035000,0.16260300)
                  rgb(0.89803922)=(0.73088900,0.87191600,0.15602900)
                  rgb(0.90196078)=(0.74138800,0.87344900,0.14956100)
                  rgb(0.90588235)=(0.75188400,0.87495100,0.14322800)
                  rgb(0.90980392)=(0.76237300,0.87642400,0.13706400)
                  rgb(0.91372549)=(0.77285200,0.87786800,0.13110900)
                  rgb(0.91764706)=(0.78331500,0.87928500,0.12540500)
                  rgb(0.92156863)=(0.79376000,0.88067800,0.12000500)
                  rgb(0.92549020)=(0.80418200,0.88204600,0.11496500)
                  rgb(0.92941176)=(0.81457600,0.88339300,0.11034700)
                  rgb(0.93333333)=(0.82494000,0.88472000,0.10621700)
                  rgb(0.93725490)=(0.83527000,0.88602900,0.10264600)
                  rgb(0.94117647)=(0.84556100,0.88732200,0.09970200)
                  rgb(0.94509804)=(0.85581000,0.88860100,0.09745200)
                  rgb(0.94901961)=(0.86601300,0.88986800,0.09595300)
                  rgb(0.95294118)=(0.87616800,0.89112500,0.09525000)
                  rgb(0.95686275)=(0.88627100,0.89237400,0.09537400)
                  rgb(0.96078431)=(0.89632000,0.89361600,0.09633500)
                  rgb(0.96470588)=(0.90631100,0.89485500,0.09812500)
                  rgb(0.96862745)=(0.91624200,0.89609100,0.10071700)
                  rgb(0.97254902)=(0.92610600,0.89733000,0.10407100)
                  rgb(0.97647059)=(0.93590400,0.89857000,0.10813100)
                  rgb(0.98039216)=(0.94563600,0.89981500,0.11283800)
                  rgb(0.98431373)=(0.95530000,0.90106500,0.11812800)
                  rgb(0.98823529)=(0.96489400,0.90232300,0.12394100)
                  rgb(0.99215686)=(0.97441700,0.90359000,0.13021500)
                  rgb(0.99607843)=(0.98386800,0.90486700,0.13689700)
                  rgb(1.00000000)=(0.99324800,0.90615700,0.14393600)},
}

\newcommand{\bdv}[2][] {\todo[inline,backgroundcolor=blue!20!white, #1]{(Bert) #2}}
\newcommand{\tvdl}[2][] {\todo[inline,backgroundcolor=red!20!white, #1]{(Thijs) #2}}
\newcommand{\ml}[2][] {\todo[inline,backgroundcolor=green!20!white, #1]{(Mykola) #2}}
\newcommand{\woutern}[2][] {\todo[inline,backgroundcolor=orange!20!white, #1]{(WouterN) #2}}

\usepackage{hyperref}

\newcommand{\refappx}[1]{\hyperref[#1]{Appendix~\ref*{#1}}}
\newcommand{\capsecref}[1]{\hyperref[#1]{Section~\ref*{#1}}}

\newif\ifanonymous
\anonymousfalse

\title{A~Message~Passing~Realization of Expected~Free~Energy~Minimization}
 
\ifanonymous
\author{Anonymous Authors}

\authorrunning{Anonymous Authors}
\institute{Anonymous Affiliation}

\else
\author{Wouter W. L. Nuijten \inst{1} \and
Mykola Lukashchuk \inst{1} \and
Thijs van de Laar \inst{1} \and Bert de Vries\inst{1,2}}

\authorrunning{W. W. L. Nuijten et al.}
\institute{Eindhoven University of Technology, 5612 AP Eindhoven, the Netherlands \and  GN Hearing, 5612 AB Eindhoven, The Netherlands}
\fi
\begin{document}

\maketitle

\begin{abstract}
    We present a message passing approach to Expected Free Energy (EFE) minimization on factor graphs, based on the theory introduced in \cite{devries_expected_2025}.
    By reformulating EFE minimization as Variational Free Energy minimization with epistemic priors,
    we transform a combinatorial search problem into a tractable inference problem solvable through standard variational techniques.
    Applying our message passing method to factorized state-space models enables efficient policy inference.
    We evaluate our method on environments with epistemic uncertainty:
    a stochastic gridworld and a partially observable Minigrid task.
    Agents using our approach consistently outperform conventional KL-control agents on these tasks,
    showing more robust planning and efficient exploration under uncertainty.
    In the stochastic gridworld environment, EFE-minimizing agents avoid risky paths, while in the partially observable minigrid setting,
    they conduct more systematic information-seeking.
    This approach bridges active inference theory with practical implementations, providing empirical evidence for the efficiency of epistemic priors in artificial agents.

\end{abstract}
\section{Introduction}\label{sec:introduction}

Expected Free Energy (EFE) minimization, rooted in the Free Energy Principle, provides a framework for modeling intelligent behavior by unifying
reward-seeking (pragmatic) and information-seeking (epistemic) drives \cite{friston_freeenergy_2010,friston_active_2015}.
While control-as-inference approaches have made significant advances in formulating decision-making as probabilistic inference
problems \cite{ito_kullback_2022,attias_planning_2003}, EFE minimization extends this paradigm by explicitly accounting for epistemic uncertainty \cite{dacosta_active_2020}, though its practical application faces computational challenges for extended planning horizons and high-dimensional state-spaces \cite{paul_active_2021}.

Traditional approaches to computing EFE often involve evaluating all possible action sequences, which becomes intractable for non-trivial problems.
While various approximations have been developed to address this tractability issue, traditional approaches typically use EFE as a cost function for evaluating policies, rather than as an objective functional for variational optimization of beliefs \cite{paul_efficient_2024,champion_branching_2022,friston_active_2025}.

This paper provides empirical validation of the theoretical foundation presented in \cite{devries_expected_2025}, which reformulates EFE minimization directly as a variational inference problem on factor graphs.
By introducing appropriate epistemic priors, we show that minimizing EFE can be achieved through standard Variational Free Energy (VFE) minimization,
making it consistent with the Free Energy Principle's core tenet that all processes are fundamentally based on variational free energy minimization.

We implement this approach through an iterative message passing algorithm on factorized state-space models.
We evaluate its performance in environments with different uncertainty characteristics: a stochastic gridworld with
perilous transitions and a partially observable Minigrid environment requiring active exploration for successful completion.
Our results confirm that agents using our inference-based method exhibit the same characteristic advantages over KL-control agents
as direct EFE computation, particularly in handling epistemic uncertainty.
This validates our approach while providing a computationally efficient framework for planning under uncertainty.

The remainder of this paper is organized as follows:
\begin{itemize}
    \item \capsecref{sec:background} provides background on necessary materials.
    \item \capsecref{sec:relatedwork} discusses related work in control as inference and active inference.
    \item \capsecref{sec:methodology} presents our methodology for reformulating EFE minimization as an inference problem.
    \item \capsecref{sec:evaluation} describes our evaluation environments and experimental design.
\end{itemize}
 
\section{Background}\label{sec:background}

\subsection{Variational Inference}
Variational inference (VI) provides a principled framework for approximating complex posterior distributions in Bayesian models \cite{koller_probabilistic_2009,bishop_pattern_2006,blei_variational_2017,winn_variational_2005}. The central challenge in Bayesian inference is computing the posterior distribution $p(\bm{x}|\bm{y})$ of hidden state sequence $\bm{x}$ given an observed data sequence $\bm{y}$, which requires evaluating the model evidence $p(\bm{y})$ \cite{cox_probability_1946,jaynes_probability_2003}. This normalization constant is typically intractable for complex models.  %\tvdl{We use a boldscript notation to indicate vectors and sequences.}

VI reformulates inference as an optimization problem by approximating the Bayesian posterior with a simpler, tractable distribution $q(\bm{x})$ from a family of distributions $\mathcal{Q}$ \cite{blei_variational_2017}. The functional we will minimize is the Variational Free Energy (VFE). The VFE is defined as $F[q] = D_{KL}(q(\bm{x}) \| p(\bm{x}|\bm{y})) - \log p(\bm{y})$, making it clear that minimizing the VFE is equivalent to minimizing the KL divergence since $\log p(\bm{y})$ is constant with respect to $q$. The VFE also provides a tractable upper bound on the negative log evidence, with $F[q] \geq -\log p(\bm{y})$ \cite{kingma_autoencoding_2022}. 

% In practice, VI often assumes a factorized form for the variational distribution, such as $q(\bm{x}) = \prod_i q(x_i)$ in mean-field approximations \cite{winn_variational_2005}. This factorization, along with appropriate parameterizations, enables efficient optimization even in high-dimensional settings. The result is a posterior approximation that balances a good fit to the data with computational tractability, providing a foundation for approximate Bayesian inference across diverse applications \cite{blei_variational_2006}.

\subsection{Factor Graphs}
Factor graphs are a specific type of probabilistic graphical model that explicitly represents the factorization structure of the model, where factors represent (conditional) probability distributions. In our work, we employ Forney-style factor graphs (FFGs) \cite{forney_codes_2001}, which offer a specific representation approach with notation following \cite{loeliger_introduction_2004}.

An FFG represents a factorized function $f(\bm{s})$ as \begin{equation}
    f(\bm{s}) =\prod_{a\in\mathcal{V}} f_a(\bm{s}_a),
\end{equation}
where $\bm{s}$ encompasses all variables in the model, and $\bm{s}_a \subseteq \bm{s}$ represents the subset of variables that participate in factor $f_a$.

In the FFG representation, nodes ($a \in \mathcal{V}$) correspond to factors in the model, while edges ($\mathcal{E} \subseteq \mathcal{V} \times \mathcal{V}$) represent variables. An edge connects to a node precisely when the variable appears as an argument in the corresponding factor. We denote the set of edges connected to node $a \in \mathcal{V}$ as $\mathcal{E}(a)$, and the nodes connected to edge $i \in \mathcal{E}$ as $\mathcal{V}(i)$.

To illustrate, the FFG representation of the factorized function
\begin{equation}\label{eq:ffg:examplefactorization}
    f(s_1,s_2,s_3,s_4) = f_a(s_1) f_b(s_1,s_2) f_c(s_3) f_d(s_2,s_3,s_4)
\end{equation}
in shown in Figure~\ref{fig:methods:example_ffg}.

\begin{figure}[tb]
    \centering
    % \begin{tikzpicture}
  
%     % nodes
%     \node[box] (fa) {$f_a$};
%     \node[box, right of=fa, node distance=20mm] (fb) {$f_b$};
%     \node[box, right of=fb, node distance=30mm] (fd) {$f_d$};
%     \node[box, below of=fd, node distance=17mm] (fc) {$f_c$};
%     \node[right of=fd, node distance=15mm] (y) {};

%     % dashed boxes
%     \node [fit = (fa), draw, inner sep=1.5mm, dashed] (fa-box) {};
%     \node [fit = (fa-box) (fb), draw, inner sep=1.5mm, dashed] (fb-box) {};
%     \node [fit=(fc), draw, inner sep=1.5mm, dashed] (fc-box) {};
%     \node [fit=(fd) (fc-box) (y) , draw, inner sep=1.5mm, dashed] (fd-box) {};
    
%     % edges with messages
%     \draw[->] (fa) -- (fb)
%         node[pos=0.5, above] {$s_1$}
%         node[pos=0.25, below] {$\rightarrow{}$}
%         node[pos=0.45, below=2.4mm, fill=white, inner sep=2pt] {\scriptsize $\vec{\mu}_{s_1}(s_1)$};
        
%     \draw[->] (fb) -- (fd)
%         node[pos=0.5, above] {$s_2$}
%         node[pos=0.25, below] {$\rightarrow{}$}
%         node[pos=0.25, below=2.4mm, fill=white, inner sep=2pt] {\scriptsize $\vec{\mu}_{s_2}(s_2)$}
%         node[pos=0.75, below] {$\leftarrow{}$}
%         node[pos=0.76, below=2.4mm, fill=white, inner sep=2pt] {\scriptsize $\cev{\mu}_{s_2}(s_2)$};
        
%     \draw[->] (fd) -- (fc)
%         node[pos=0.5, left] {$s_3$}
%         node[pos=0.6, right=3mm, fill=white, inner sep=2pt] {\scriptsize $\cev{\mu}_{s_3}(s_3)$}
%         node[pos=0.6, right] {$\uparrow{}$};
    
%     % edges without messages
%     \draw[->] (fd) -- (y)
%         node[pos=0.5, above] {$s_4$};

% \end{tikzpicture}

\begin{tikzpicture}
  
    % nodes
    \node[box] (fa) {$f_a$};
    \node[box, right of=fa, node distance=20mm] (fb) {$f_b$};
    \node[box, right of=fb, node distance=35mm] (fd) {$f_d$};
    \node[box, below of=fd, node distance=20mm] (fc) {$f_c$};
    \node[right of=fd, node distance=20mm] (y) {};

    \draw[-] (fa) -- (fb)
        node[pos=0.5, above] {$s_1$};
    \draw[-] (fb) -- (fd)
        node[pos=0.5, above] {$s_2$};
    \draw[-] (fd) -- (fc)
        node[pos=0.5, left] {$s_3$};
    \draw[-] (fd) -- (y)
        node[pos=0.5, above] {$s_4$};

    % % dashed boxes
    \node [fit = (fa), draw=none, inner sep=1.5mm] (fa-box) {};
    \node [fit = (fa-box) (fb), draw, inner sep=1.5mm, dashed] (fb-box) {};

    \node [fit=(fc), draw, inner sep=1.5mm, dashed] (fc-box) {};
    \node [fit = (fd) (fc-box) (y), draw, inner sep=1.5mm, dashed] (fcdy-box) {};
    % \node [fit=(fd) (fc-box) (y) , draw, inner sep=1.5mm, dashed] (fd-box) {};
    
    % Add arrow from the dashed box
    \draw[->] ([yshift=-5mm]fb-box.east) -- ++(10mm,0) node[below=0cm, pos=0.5] {$\overrightarrow{\mu}(s_2)$};
    \draw[->] ([yshift=5.7mm]fcdy-box.west) -- ++(-10mm,0) node[below=0cm, pos=0.5] {$\overleftarrow{\mu}(s_2)$};
    \draw[->] ([xshift=2mm]fc-box.north) -- ++(0,8mm) node[right=0cm, pos=0.5] {$\overrightarrow{\mu}(s_3)$};

    % % edges with messages
    % \draw[->] (fa) -- (fb)
    %     node[pos=0.5, above] {$s_1$}
    %     node[pos=0.25, below] {$\rightarrow{}$}
    %     node[pos=0.45, below=2.4mm, fill=white, inner sep=2pt] {\scriptsize $\vec{\mu}_{s_1}(s_1)$};
        
    % \draw[->] (fb) -- (fd)
    %     node[pos=0.5, above] {$s_2$}
    %     node[pos=0.25, below] {$\rightarrow{}$}
    %     node[pos=0.25, below=2.4mm, fill=white, inner sep=2pt] {\scriptsize $\vec{\mu}_{s_2}(s_2)$}
    %     node[pos=0.75, below] {$\leftarrow{}$}
    %     node[pos=0.76, below=2.4mm, fill=white, inner sep=2pt] {\scriptsize $\cev{\mu}_{s_2}(s_2)$};
        
    % \draw[->] (fd) -- (fc)
    %     node[pos=0.5, left] {$s_3$}
    %     node[pos=0.6, right=3mm, fill=white, inner sep=2pt] {\scriptsize $\cev{\mu}_{s_3}(s_3)$}
    %     node[pos=0.6, right] {$\uparrow{}$};
    
    % % edges without messages
    % \draw[->] (fd) -- (y)
    %     node[pos=0.5, above] {$s_4$};

\end{tikzpicture}
    \caption{A Forney-style factor graph representation of the factorization in \eqref{eq:ffg:examplefactorization}.}
    \label{fig:methods:example_ffg}
\end{figure}
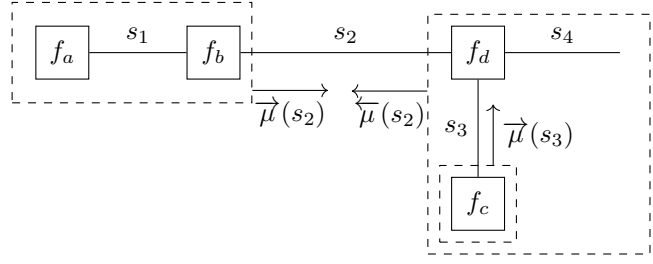

A common approach to realizing efficient variational inference on factor graphs involves the Bethe assumption, which posits that the posterior distribution factorizes as a product of local marginals associated with the nodes and edges of the graph. This structural assumption on the posterior distributions enables the formulation of message passing algorithms that seek out stationary points of the Bethe free energy \cite{yedidia_constructing_2005,senoz_message_2022,dauwels_variational_2007}.

% Message passing algorithms, particularly the sum-product algorithm \cite{pearl_reverend_1982,loeliger_factor_2007}, exploit the factorization structure of factor graphs to dramatically reduce computational complexity during inference. By applying the distributive law to factorized joint distributions, these algorithms transform global inference problems into localized computations that pass messages (beliefs) between adjacent nodes. This approach reduces the exponential complexity of direct marginalization to operations that scale with the size of the local factors. This computational efficiency enables tractable inference in high-dimensional models that would otherwise be computationally prohibitive.

To illustrate the computational benefits of message passing, consider the generative model in \eqref{eq:ffg:examplefactorization}, and assume we are interested in computing $p(s_2)$. This marginal distribution can be obtained by summing out all other variables from the joint
\begin{equation}
    p(s_2) = \sum_{s_1} \sum_{s_3} \sum_{s_4} f(s_1, s_2, s_3, s_4) \,,
    \label{eq:ffg:fullmarginal}
\end{equation}
which, when each $s_i$ can take 10 values, contains about a thousand terms. However, taking into account the factorization of the generative model and the distributive law of the product, \eqref{eq:ffg:fullmarginal} can be rewritten as
\begin{equation}
    p(s_2) = \bigg(\underbrace{\sum_{s_1} f_a(s_1)f_b(s_1,s_2)}_{\overrightarrow{\mu}(s_2)} \bigg) \cdot \bigg(\underbrace{\big(\overbrace{\sum_{s_3}f_c(s_3)}^{\overrightarrow{\mu}(s_3)}\big)\sum_{s_4}f_d(s_2,s_3,s_4)}_{\overleftarrow{\mu}(s_2)}\bigg) \,.
    \label{eq:ffg:messagepassing}
\end{equation}
The computation in \eqref{eq:ffg:messagepassing} requires only a few hundred summations and is preferable from a computational standpoint. In larger models, the number of computations scale linearly with the number of factor nodes, instead of exponentially. The intermediate results $\overrightarrow{\mu}(s_i)$ and $\overleftarrow{\mu}(s_i)$ afford an interpretation as local message in the FFG representation of the model, see \autoref{fig:methods:example_ffg}.
\woutern{Is this enough on message passing or should I add math?}
\bdv{I think you should add an example that show how MP follows from the distributive law applied to inference for $s_2$ in Fig.~\ref{fig:methods:example_ffg}.}
For comprehensive treatments of factor graphs and associated (variational) message passing algorithms, we refer readers to \cite{loeliger_introduction_2004,loeliger_factor_2007,yedidia_constructing_2005,dauwels_variational_2007,senoz_message_2022}.

\section{Related Work}\label{sec:relatedwork}
Autonomous decision-making under uncertainty remains a central challenge in control theory and artificial intelligence. This section reviews key developments that contextualize our contribution.

\subsection{Control as Inference}
The pursuit of efficient and high-performing autonomous systems has driven significant research in control theory. Optimal control \cite{bellman_theory_1954,bellman_dynamic_1966,pontryagin_mathematical_2018} provides a mathematical framework for determining the control inputs that minimize a predefined cost function for a given system. Building upon these foundations, Model Predictive Control (MPC) algorithms address the challenges of real-time control by incorporating a feedback loop and a receding horizon strategy \cite{bertsekas_dynamic_2012,richalet_algorithmic_1976,richalet_model_1978,cutler_dynamic_1979}. This approach allows for online adaptation to disturbances and constraints.

A significant paradigm shift in recent years involves viewing control as an inference problem. This perspective allows the application of powerful probabilistic tools to address control challenges, particularly in complex and uncertain environments. Under deterministic dynamics, the sequential decision-making process in closed-loop receding horizon MPC can be elegantly mapped to inference on a factor graph \cite{levine_reinforcement_2018,lazaro-gredilla_what_2024}. % In this framework, the backward propagation of messages naturally mirrors the Bellman backups, where the optimal value function is computed recursively.

When dealing with stochastic dynamics or the need for state estimation under uncertainty, stochastic optimal control methods can be reformulated using variational inference \cite{kappen_optimal_2012,ito_kullback_2022}. Here, the intractable posterior distribution over states and/or controls is approximated by a tractable variational distribution.

Active inference \cite{dacosta_active_2020,dacosta_active_2024} addresses control under uncertainty by proposing that information gained about the system is also a form of reward. The framework suggests that variational inference naturally balances exploration and exploitation by optimizing the Expected Free Energy \cite{friston_active_2015}, which elegantly combines the drive to minimize uncertainty about the environment (information gain) with the need to achieve desired outcomes. However, a current limitation of active inference lies in the computational cost associated with computing the Expected Free Energy \cite{friston_active_2015}, which has spurred recent research into efficient algorithms \cite{paul_efficient_2024,friston_sophisticated_2021,paul_active_2021,champion_branching_2022}.

Recently, \cite{devries_expected_2025} proposed an alternative approach to Expected Free Energy minimization by framing EFE minimization as a regular variational free energy minimization task. This approach is promising for scalable implementation of EFE-minimizing planning algorithms, but offers a theoretical account, without considering practical implementation or empirical validation. In the next section, we will propose a message passing realization of this approach.

\section{Methodology}\label{sec:methodology}

For the main contribution of this paper, we will elaborate on Theorem 1 from \cite{devries_expected_2025}. For convenience, we will repeat the theorem here, albeit without the inclusion of model parameters $\theta$:
\begin{theorem}[Expected Free Energy Theorem]\label{the:efetheorem}
    Consider an agent with generative model $p(\bm{y}, \bm{x}, \bm{u})$, and prior beliefs $\hat{p}(\bm{x})$ about future desired states.

    Consider the Variational Free Energy functional
    \begin{equation}\label{eq:VFE-for-planning}
        F[q] \triangleq E_{q(\bm{y}, \bm{x}, \bm{u})}\bigg[ \log \frac{ \overbrace{q(\bm{y}, \bm{x}, \bm{u})}^{\text{posterior}} }{ \underbrace{p(\bm{y}, \bm{x}, \bm{u})}_{\substack{\text{generative} \\ \text{model}} } \underbrace{\hat{p}(\bm{x})}_{\substack{\text{preference}\\ \text{prior}}} \underbrace{\tilde{p}(\bm{u}) \tilde{p}(\bm{x})}_{\substack{\text{epistemic} \\ \text{priors}}}} \bigg]\,,
    \end{equation}
    where the generative model in the denominator is augmented by both a preference prior $\hat{p}(\cdot)$ and epistemic priors $\tilde{p}(\cdot)$.

    If the epistemic priors are chosen as
    \begin{subequations}\label{eq:epistemic-priors}
        \begin{align}
            \tilde{p}(\bm{u}) & \propto \exp(H[q(\bm{x}|\bm{u})]) \label{eq:epistemic-prior-u}  \\
            \tilde{p}(\bm{x}) & \propto \exp(-H[q(\bm{y}|\bm{x})]) \label{eq:epistemic-prior-x}
        \end{align}
    \end{subequations}

    then $F[q]$ decomposes as
    \begin{align}\label{eq:F=G+complexity}
        F[q] = \underbrace{E_{q(\bm{u})}\big[ G(\bm{u})\big]}_{\substack{ \text{expected policy} \\ \text{costs} }}  + \underbrace{E_{q(\bm{y}, \bm{x}, \bm{u})}\bigg[\log \frac{q(\bm{y}, \bm{x}, \bm{u})}{p(\bm{y}, \bm{x}, \bm{u})}\bigg]}_{\text{complexity}}  + \mathrm{constant}\,,
    \end{align}
    where
    \begin{equation}
        G(\bm{u}) =  \mathbb{E}_{q(\bm{y}, \bm{x}|\bm{u})}\bigg[\log\bigg( \frac{q(\bm{x}|\bm{u})}{\hat{p}(\bm{x})}\cdot \frac{1}{q(\bm{y}|\bm{x})} \bigg) \bigg]
    \end{equation} is the expected free energy as defined in \cite{dacosta_active_2020}.
    In \eqref{eq:epistemic-priors},
    \begin{equation}
        H[q(y|x)] = -\int q(y|x) \log q(y|x) \text{d}y
    \end{equation} is the entropy functional.
\end{theorem}
\begin{proof}
    The proof of \eqref{eq:F=G+complexity} is given in \cite[Appendix A]{devries_expected_2025}.
\end{proof}
While \eqref{eq:F=G+complexity} shows that minimization of the $F[q]$ leads to minimization of (expected) $G(u)$, the proof of \eqref{eq:F=G+complexity} is declarative and does not provide an explicit algorithm for minimizing $F[q]$.

In the following sections, we will describe a message passing algorithm on factor graphs that can be used as a practical approach to search for stationary points of the free energy functional.
\subsection{Factorized models and factorized posteriors}

\autoref{the:efetheorem} is a general result, however, in practice, we are often interested in factorized state-space models of the form
\begin{equation} \label{eq:factorized-model}
    p(\bm{y}, \bm{x}, \bm{u}) = p(x_0) \prod_{t=1}^{T} p(y_t | x_t) p(x_t | x_{t-1}, u_t)  p(u_t)
\end{equation}
We can make an additional assumption that the posterior distribution factorizes in the same way as the generative model:
\begin{equation} \label{eq:factorized-posterior}
    q(\bm{y}, \bm{x}, \bm{u}) = q(x_0) \prod_{t=1}^{T} q(y_t | x_t) q(x_t | x_{t-1}, u_t)  q(u_t)\,.
\end{equation}
Note that this is consistent with making the Bethe assumption, which says that the variational posterior distribution can be decomposed into local contributions:
\begin{equation}
    q(\bm{s}) = \prod_{a \in \mathcal{V}}q_a(\bm{s}_a) \prod_{i\in \mathcal{E}}q_i(s_i)^{-1}
\end{equation}
for $\mathcal{G} = (\mathcal{V}, \mathcal{E})$ the underlying FFG.
% \tvdl{Is this just the Bethe factorization that is implied by the GM factorization? Or is there something else going on? A formulaton in terms of a Bethe factorization would make the result more broadly applicable. Perhaps the corollary is even immediate under the Bethe assumption, since all terms become local.}
% \woutern{Yes, it is the Bethe assumption on the variational posterior (which is not implied by GM factorization). I am a bit hesitant to note this here, Bert advised me against mentionin the Bethe assumption.}
% \bdv{I am not against mentioning the Bethe assumption, but are you sure that this is equivalent to the Bethe assumption. I thought the Bethe assumption was something like $$q(x) = \prod_a q(x_a) \prod_i q(x_i)^{-1}$$}
% \woutern{I'm sure, the edge divisions are just taken in by the conditional distributions ($\frac{q(x_{t}, x_{t-1}, u_t)}{q(u_t)q(x_{t-1})}$ replace by $q(x_t|x_{t-1}, u_t)$ and the division of $q(x_t)$ getting absorbed by the next timestep). Should I say "Note that this is equivalent to making the Bethe assumption"? Would this resolve the issue?}
Under this assumption, we can derive a corollary to \autoref{the:efetheorem} that provides more specific expressions for the epistemic priors.

\begin{corollary} \label{cor:bethe}
    Consider an agent with Variational Free Energy functional as in \eqref{eq:VFE-for-planning}, comprising a generative model \eqref{eq:factorized-model}, a posterior distribution factorized as in \eqref{eq:factorized-posterior}, and a preference prior $\hat{p}(\bm{x}) = \prod_{t=1}^{T} \hat{p}(x_t)$.
    If the epistemic priors are chosen as    \begin{subequations}\label{eq:epistemic-priors-bethe}
        \begin{align}
            \tilde{p}(u_t) & \propto \exp(H[q(x_t, x_{t-1} | u_t)] - H[q(x_{t-1} | u_t)]) \label{eq:epistemic-prior-u-bethe} \\
            \tilde{p}(x_t) & \propto \exp(-H[q(y_t | x_t)]) \label{eq:epistemic-prior-x-bethe}
        \end{align}
    \end{subequations}
    then the Variational Free Energy functional \eqref{eq:VFE-for-planning} decomposes as
    \begin{align}
        F[q] = E_{q(\bm{u})}\big[ G(\bm{u})\big]  + E_{q(\bm{y}, \bm{x}, \bm{u})}\bigg[\log \frac{q(\bm{y}, \bm{x}, \bm{u})}{p(\bm{y}, \bm{x}, \bm{u})}\bigg] \, + \mathrm{constant}.
    \end{align}
\end{corollary}
While this corollary is a special case and a direct application of \autoref{the:efetheorem}, an elaboration of the proof is given in \refappx{sec:bethe-proof}.
% \tvdl{If this is just the application of the original theorem to a specific model, then why is a separate corollary needed? Isn't this just a straightforward result under a model substitution and Bethe approximation?}
% \woutern{This corollary might be trivial to people working with factor graphs and BFE every day, but not for the average reader. The corollary also introduces what to do with nodes that have multiple interfaces.}
% \bdv{agreed w Wouter.}
This corollary states that the preference and epistemic priors can be reduced to local contributions. We will implement the preference and epistemic priors as factor nodes that act as prior distributions during the inference procedure. A timeslice of the augmented factor graph is shown in \autoref{fig:augmented-fg}.

\begin{wrapfigure}{R}{0.4\textwidth}
    \centering
    \includegraphics[width=\linewidth]{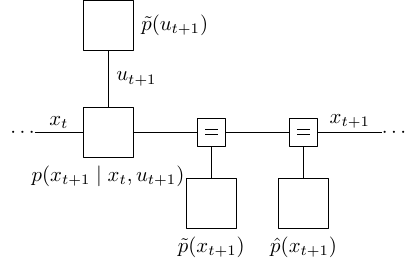}
    \caption{Slice of the factor graph representation of the augmented generative model. The original generative model \eqref{eq:factorized-model} is augmented with epistemic priors $\tilde{p}(u_{t+1})$ and $\tilde{p}(x_{t+1})$, and preference priors $\hat{p}(x_{t+1})$ for future timesteps.}
    \label{fig:augmented-fg}
\end{wrapfigure}
% \tvdl{I think it looks nicer if the distribution annotations in Fig.~2 are placed outside the node. Alternatively, the nodes can be inscribed with just $\tilde{p}$ or $\hat{p}$, since the associated variable is implied by the connected edge.} \bdv{agreed.}
% \bdv{you cannot have the same variable label $x_{t+1}$ on two different edges. I also agree that it is better to put the node label next to (rather than inside) the node. }
The benefit of this approach is that inference on factor graphs is well-understood and can be implemented efficiently using reactive message passing \cite{bagaev_reactive_2021}.
Effectively, this means that the computational complexity of Expected Free Energy minimization is the same as the computational complexity of variational inference on a factor graph.

\subsection{Inferring a policy posterior}
\autoref{cor:bethe} introduces a circular dependency in the model definition: to define the VFE functional with epistemic priors \eqref{eq:epistemic-priors-bethe}, we need access to the variational posterior distribution, but the variational posterior can only be obtained by minimizing the VFE functional given the generative model.

This circular dependency can be resolved through an iterative variational inference procedure implemented as message passing on the factor graph. We first initialize the variational posterior and then iteratively update both the posterior beliefs and epistemic priors until convergence.

On a factor graph, we can implement variational inference using message passing algorithms that iteratively updates posterior distributions \cite{pearl_reverend_1982}.

Each message passing iteration $\tau$ refines both the posteriors and priors simultaneously.
To that extent, let $q_\tau(\cdot)$ be the variational posterior distribution at iteration $\tau$, we then define the epistemic priors as
\begin{equation}
    \begin{aligned}
        \tilde{p}_\tau(u_t) & = \sigma\big(H[q_{\tau-1}(x_t, x_{t-1} | u_t)] - H[q_{\tau-1}(x_{t-1} | u_t)]\big) \\
        \tilde{p}_\tau(x_t) & = \sigma\big(-H[q_{\tau-1}(y_t | x_t)]\big)\,.
    \end{aligned}
\end{equation}
Here, $\sigma$ is the softmax function, which guarantees proportionality as in Equations \ref{eq:epistemic-prior-u-bethe} and \ref{eq:epistemic-prior-x-bethe}.
A formal description of the algorithm is given in \autoref{alg:efevfe}.
While this approach solves the initialization problem, there are some subtleties that need to be addressed. Specifically, although the subtraction of entropies in line \autoref{eq:Cu-after-entropies-substitution} results in a constant when using the same variational distribution $q$ for both the epistemic prior $\tilde{p}$ and the optimization, this property no longer holds when we use different distributions - namely, when we use $q_{\tau-1}$ to define $\tilde{p}_\tau$ but optimize with respect to $q_\tau$.
While this is not a problem if the inference procedure converges, this convergence is not guaranteed.
\begin{algorithm}
    \caption{EFE minimization as VFE minimization}
    \begin{algorithmic}[1]
        \State \textbf{Input:} Factorized generative model $p(\bm{y}, \bm{x}, \bm{u})$, preference prior $\hat{p}(\bm{x})$, number of VI iterations $\tau_{max}$
        \State \textbf{Output:} Policy posterior $q_{\tau_{max}}(\bm{u})$
        \State $q_0(\bm{y}, \bm{x}, \bm{u}) \gets $ Uninformative distribution
        \For{$\tau \gets 1$ to $\tau_{max}$} \Comment{Iterations of variational inference algorithm}
        \For{each time step $t$}
        \State $\tilde{p}_\tau(u_t) \gets \sigma(H[q_{\tau-1}(x_t, x_{t-1} | u_t)] - H[q_{\tau-1}(x_{t-1} | u_t)])$
        \State $\tilde{p}_\tau(x_t) \gets \sigma(-H[q_{\tau-1}(y_t | x_t)])$
        \EndFor
        \State $q_\tau(\bm{y}, \bm{x}, \bm{u}) \gets \mathrm{infer}(p(\bm{y},\bm{x},\bm{u}))$ \Comment{Message passing \eqref{eq:ffg:messagepassing}}
        \EndFor
        \State \Return $q_{\tau_{max}}(\bm{u})$
    \end{algorithmic}
    \label{alg:efevfe}
\end{algorithm}
\ml{I see here two ways to write comment    s and this is strange. You either always use triangle or everytime \{Do x:\}. Also the 3 line is state assigning, it's not a comment, but then it means you need to have an arrow there. The same goes for the line 10, it's also an assigment, so it should be done with an arrow.}
\bdv{You write at top you want to return $q^*(u)$, but at bottom you return $q_\tau(u)$. Also, in line 10 I would write updating the marginals $q_\tau(u)$, $q_\tau(x)$ etc on the edges rather than the joint $q_\tau(y,x,u)$.}
\woutern{I am a bit afraid that talking about edge-marginals introduces ambiguity in two ways: (1) we have to talk about that $q_\tau(x_t)$ gets updated for all $t$ (which can get messy), and (2) lines 7 and 8 explicitly mention node-local joint marginals, which I am afraid is not clear that they will be updated if we only talk about variable marginals.}
\tvdl{If all computations in the algorithm are local, it can be visualized as message passing on the factor graph. Then the procedure can be formulated in terms of local messages around an epistemic node, which makes the implementation modular and scalable (in contrast to an algorithmic presentation which introduces assumptions that are not strictly required). See e.g. chance-constrained active inference, https://arxiv.org/pdf/2102.08792}
\tvdl{The paper is on message passing, but I see no messages anywhere.}
\section{Evaluation}\label{sec:evaluation}

% This section evaluates our EFE-minimizing policy inference method. The addition of preference priors is consistent with the literature on KL control \cite{todorov_linearly-solvable_2006,todorov_general_2008}, which means the main point of interest is the influence of the epistemic priors on the policy posterior. We compare our EFE-minimizing approach against a standard KL-control policy by running experiments with and without the epistemic priors.

This section evaluates our EFE-minimizing policy inference method.
In this section, we will evaluate the performance of the proposed method. The addition of preference priors is consistent with the literature on KL control \cite{todorov_linearlysolvable_2006,todorov_general_2008}, which means the main point of interest is the influence of the epistemic priors on the policy posterior.
To this extent, we will execute the experiments both with and without the epistemic priors, which will correspond to a KL-control and an EFE-minimizing policy, respectively. \tvdl{There is probably a Friston paper you can refer to here.} \woutern{Why? I'm not implementing a friston paper here, this is our method.}
KL-control is known to be prone to optimistic planning in the face of stochasticity and uncertainty \cite{lazaro-gredilla_what_2024,levine_reinforcement_2018}, so we will explore partially observable Markov decision processes (POMDPs) with stochastic dynamics and observation noise.

For our experimental evaluation, we consider scenarios where the environment dynamics are completely known to the agent, though they may be stochastic or contain inherent uncertainty. This known-dynamics assumption allows us to isolate and evaluate the specific effects of epistemic priors on decision-making, without conflating them with model learning. %The agent has access to the complete transition and observation models, but still faces epistemic uncertainty due to stochasticity in transitions and partial observability, requiring it to maintain and update beliefs about its current state.

\subsection{Experimental design}

We designed a stochastic grid environment that specifically challenges agents with uncertainty in dynamics and observations. Additionally, we evaluate our method on the Minigrid door-key environment \cite{chevalier-boisvert_minigrid_2023}, which tests how agents handle partial observability. Both environments highlight the differences between KL-control and EFE-minimizing policies in the presence of epistemic uncertainty.

\subsubsection{Stochastic Grid Environment}
For our first experiment, we focus on a stochastic grid environment. In this environment, the agent has to traverse the grid from one end to the other, with hazards and stochastic transitions. The key challenge is that on the shortest path from the start to the goal, there are cells in which the transition matrix is stochastic, with the risk that the agent will end up in a sink state.
\begin{figure}[btp]
    \centering

    \resizebox{0.6\textwidth}{!}{\input{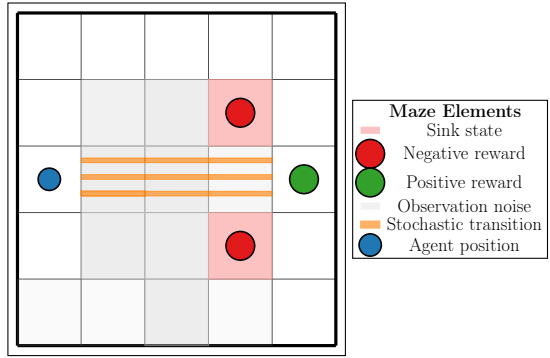}}

    \caption{The stochastic grid environment. The agent should traverse the grid with both stochastic transitions and observation noise. Cells with stochastic transitions appear on the shortest path, creating a risk-reward tradeoff. Opacity for observation noise is used to indicate the uncertainty in the environment.}
    \label{fig:stochastic_maze}
\end{figure}
The stochasticity presents a direct test of how agents handle uncertainty in dynamics: the KL-control agent is expected to plan optimistically through these uncertain transitions, while the EFE-minimizing agent should recognize the epistemic risk and avoid these cells. This environment also features observation noise, adding another layer of uncertainty that forces the agent to maintain beliefs over possible states rather than having full observability.

A longer but safer path exists that avoids all stochastic transitions. The optimal policy for a risk-aware agent would be to take this safer path, despite it requiring more steps. A visualization of the environment is shown in \autoref{fig:stochastic_maze}.

The agent receives a reward of $1$ for reaching the goal. When ending up in a sink state, the agent receives a penalty of $-1$. The full specification of the generative model can be found in \refappx{appx:gridworld-generative-model}.

\subsubsection{Minigrid Door-Key Environment}
\begin{wrapfigure}{R}{0.30\linewidth}
    \centering
    \includegraphics[width=\linewidth]{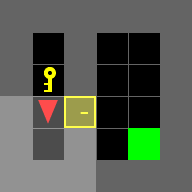}
    \caption{An initial state of the Minigrid environment. The agent has a limited field of view, indicated by the highlighted cells.}
    \label{fig:minigrid_initial}
\end{wrapfigure}
The second environment we consider is a Minigrid environment, specifically a 4x4 door-key environment. This environment tests a different aspect of epistemic uncertainty, namely, partial observability. The agent has a limited field of view, which means that the agent must actively explore to reduce uncertainty about the environment state.

The task requires the agent to locate and pick up a key, find and open a door, and finally reach the goal square. This multi-step process creates a natural exploration challenge that tests how agents handle partial observability. The agent location, key location, and door location are randomized in each episode, which means that the agent has epistemic uncertainty about the environment state.

The EFE-minimizing agent should show more directed exploration behavior, actively seeking to reduce uncertainty about the key and door locations. In contrast, the KL-control agent (without epistemic priors) might exhibit less efficient exploration patterns, as it lacks the intrinsic drive to resolve uncertainty.

The Minigrid environment adds another layer of complexity to the task, as the field of view means that the observations are relative to the agent, while the goals are formulated in an external frame of reference. This means that the observation space of the agent is much larger than the state space. The observation space is of size $\approx5^{49}$ \ml{Saying that smt is growing as a constant does not make much of a sense, you want to say probably that size of the env is approximatly such big.}, which makes algorithms like Sophisticated Inference \cite{friston_sophisticated_2021} intractable. Furthermore, the planning horizon of 22 timesteps makes standard Expected Free Energy computation as policy evaluation intractable. The computational complexity of the door-key environment is where the benefits of the proposed method are most evident.

A visualization of the initial state of the Minigrid environment is shown in \autoref{fig:minigrid_initial}. The agent receives a reward when reaching the goal, proportional to the number of steps taken. The full specification of the used generative model can be found in \refappx{appx:minigrid-generative-model}. The source code and implementation details for all experiments presented in this paper are publicly available in our online repository\footnote{\ifanonymous\url{https://anonymous.4open.science/r/EFEasVFE-DA7F}\else\url{https://github.com/biaslab/EFEasVFE}\fi}.

\subsection{Results}

\subsubsection{Stochastic Grid Environment} \label{sec:stochastic-grid-environment-results}

We evaluated the performance of both agents across 100 episodes, \autoref{tab:combined-results}, left, summarizes the quantitative results.

This table suggests distincly different navigational patterns between both agents.
The EFE-minimizing agent consistently chooses the longer but safer path around the stochastic transition cells, demonstrating risk-averse behavior that aligns with theoretical predictions. In contrast, the KL-control agent attempts the shorter path through cells with stochastic transitions, exhibiting the optimistic planning tendency typical of approaches that wrongly account for the system's aleatoric uncertainty.
A more detailed visualization of the trajectories for both agents, as well as an empirical convergence analysis of our algorithm, is provided in \refappx{appx:additional_results_gridworld}.

\subsubsection{Minigrid Door-Key Environment}

We evaluated both agents across 200 experimental episodes with a planning horizon of 25 steps. \autoref{tab:combined-results}, right, presents the quantitative comparison between the EFE-minimizing and KL-control agents in the Minigrid door-key environment.
\begin{table}[bt]

    \centering
    \begin{tabular}{l|cc|l|cc}
        \multicolumn{3}{c|}{Stochastic Grid} & \multicolumn{3}{c}{Minigrid Door-Key}                                                                                                              \\
        \hline
        Metric                               & KL                                    & EFE (ours)         & Metric         & KL                           & EFE (ours)                            \\
        \hline
        Success Rate                         & 21\%                                  & \textbf{100\%}     & Success Rate   & 85.0\%                       & \textbf{95.0\%}                       \\
        Avg. Reward                          & 0.22 ± 0.77                           & \textbf{1.00 ± 0}  & Avg. Reward    & 0.82 ± 0.35                  & \textbf{0.92 ± 0.21}                  \\
        \multirow{2}{*}{-}                   & \multirow{2}{*}{-}                    & \multirow{2}{*}{-} & Avg. Time to   & \multirow{2}{*}{2.0 ± 3.65} & \multirow{2}{*}{\textbf{1.28 ± 0.64}} \\
                                             &                                       &                    & Key Visibility &                              &                                       \\
    \end{tabular}
    \caption{Performance comparison across environments (100 episodes for Stochastic Grid, 200 episodes for Minigrid).}
    \label{tab:combined-results}

\end{table}

The EFE-minimizing agent demonstrates more effective exploration patterns, particularly in scenarios requiring active information seeking. This is especially evident in the reduced time needed to locate the key, confirming that epistemic priors enable more directed information-seeking in partially observable environments.

A more detailed visualization of the trajectories for both agents and an empirical convergence analysis of our algorithm is provided in \refappx{appx:additional_results_minigrid}.

\section{Discussion}\label{sec:discussion}

Our experimental results demonstrate that agents using the proposed message passing approach for EFE minimization exhibit the characteristic behaviors of active inference: risk-averse path selection in stochastic environments and information-seeking exploration in partially observable settings. These behaviors emerge naturally from the inclusion of epistemic priors in the variational free energy objective, without requiring explicit computation of expected free energy.

The reformulation of EFE minimization as a variational inference problem provides several advantages: it maintains theoretical consistency with the Free Energy Principle's core tenet; transforms a combinatorial search problem into a tractable inference procedure using message passing on factor graphs; and eliminates the need for ad hoc policy pruning, replacing it with principled reactive processing where the agent minimizes VFE at each point in time. This approach is particularly valuable in complex environments where traditional EFE computation becomes intractable, as demonstrated in our Minigrid experiments.

While our implementation shows promising results, the convergence properties of our iterative approach to handling self-referential epistemic priors require further theoretical investigation. Future research should investigate the inclusion of additional parameters in the generative model, particularly those related to environment dynamics. A natural extension of our work would be to incorporate parameter learning within the epistemic priors. This would allow agents to infer policies that facilitate sample-efficient learning of model parameters. This concept has already been introduced in \cite{devries_expected_2025}. However, the exact functional form of the empirical prior has not yet been derived.
\section{Conclusion}\label{sec:conclusion}

In this paper, we presented a message passing implementation of Expected Free Energy minimization on factor graphs. Our approach reframes EFE minimization as a variational inference problem, allowing us to use standard message passing algorithms for efficient policy inference. The key insight is that by introducing appropriate epistemic priors, we can transform the expected free energy objective into a modified variational free energy objective that can be optimized through standard inference techniques.

Our experimental results in both stochastic and partially observable environments demonstrate that this approach reproduces the characteristic behaviors of active inference: risk aversion in environments with hazardous stochasticity and information seeking in partially observable environments. The message passing implementation shows significant advantages in computational efficiency compared to traditional methods for computing expected free energy, particularly in complex environments with high-dimensional observation spaces and long planning horizons.

By reformulating EFE minimization as variational inference, our work contributes to unifying the theoretical frameworks of the Free Energy Principle and active inference with practical implementations for decision-making under uncertainty. This bridges the gap between theoretical accounts of intelligent behavior and efficient algorithms for artificial agents, offering a principled approach to balancing pragmatic and epistemic objectives in complex and uncertain environments.

\tvdl{Note to self: read until here (May 7).}
\ifanonymous
\else

    \section*{Acknowledgements}
    This publication is part of the project "ROBUST: Trustworthy AI-based Systems for Sustainable Growth" with project number KICH3.LTP.20.006, which is (partly) financed by the Dutch Research Council (NWO), GN Hearing, and the Dutch Ministry of Economic Affairs and Climate Policy (EZK) under the program LTP KIC 2020-2023.
\fi

\bibliographystyle{splncs04}
\bibliography{wouter}

@inproceedings{attias_planning_2003,
  title = {Planning by Probabilistic Inference},
  booktitle = {International Workshop on Artificial Intelligence and Statistics},
  author = {Attias, Hagai},
  year = 2003,
  pages = {9--16},
  publisher = {PMLR},
  urldate = {2025-05-04}
}

@misc{bagaev_reactive_2021,
  title = {Reactive {{Message Passing}} for {{Scalable Bayesian Inference}}},
  author = {Bagaev, Dmitry and {de Vries}, Bert},
  year = 2021,
  month = dec,
  number = {arXiv:2112.13251},
  eprint = {2112.13251},
  primaryclass = {cs},
  publisher = {arXiv},
  doi = {10.48550/arXiv.2112.13251},
  urldate = {2023-07-25},
  archiveprefix = {arXiv}
}

@article{bellman_dynamic_1966,
  title = {Dynamic {{Programming}}},
  author = {Bellman, Richard},
  year = 1966,
  journal = {Science},
  volume = {153},
  number = {3731},
  eprint = {1719695},
  eprinttype = {jstor},
  pages = {34--37},
  publisher = {American Association for the Advancement of Science},
  issn = {0036-8075},
  urldate = {2025-05-01}
}

@article{bellman_theory_1954,
  title = {The Theory of Dynamic Programming},
  author = {Bellman, Richard},
  year = 1954,
  journal = {Bulletin of the American Mathematical Society},
  volume = {60},
  number = {6},
  pages = {503--515},
  issn = {0002-9904, 1936-881X},
  doi = {10.1090/S0002-9904-1954-09848-8},
  urldate = {2025-05-01},
  langid = {english}
}

@book{bertsekas_dynamic_2012,
  title = {Dynamic Programming and Optimal Control: {{Volume I}}},
  shorttitle = {Dynamic Programming and Optimal Control},
  author = {Bertsekas, Dimitri},
  year = 2012,
  volume = {4},
  publisher = {Athena scientific},
  urldate = {2025-05-04}
}

@book{bishop_pattern_2006,
  title = {Pattern Recognition and Machine Learning},
  author = {Bishop, Christopher M. and Nasrabadi, Nasser M.},
  year = 2006,
  volume = {4},
  publisher = {Springer},
  urldate = {2025-05-06}
}

@article{blei_variational_2017,
  title = {Variational {{Inference}}: {{A Review}} for {{Statisticians}}},
  shorttitle = {Variational {{Inference}}},
  author = {Blei, David M. and Kucukelbir, Alp and McAuliffe, Jon D.},
  year = 2017,
  month = apr,
  journal = {Journal of the American Statistical Association},
  volume = {112},
  number = {518},
  pages = {859--877},
  publisher = {Taylor \& Francis},
  issn = {0162-1459},
  doi = {10.1080/01621459.2017.1285773},
  urldate = {2023-07-25}
}

@article{champion_branching_2022,
  title = {Branching {{Time Active Inference}}: {{The}} Theory and Its Generality},
  shorttitle = {Branching {{Time Active Inference}}},
  author = {Champion, Th{\'e}ophile and Da Costa, Lancelot and Bowman, Howard and Grze{\'s}, Marek},
  year = 2022,
  month = jul,
  journal = {Neural Networks},
  volume = {151},
  pages = {295--316},
  issn = {0893-6080},
  doi = {10.1016/j.neunet.2022.03.036},
  urldate = {2025-05-06}
}

@article{chevalier-boisvert_minigrid_2023,
  title = {Minigrid \& Miniworld: {{Modular}} \& Customizable Reinforcement Learning Environments for Goal-Oriented Tasks},
  shorttitle = {Minigrid \& Miniworld},
  author = {{Chevalier-Boisvert}, Maxime and Dai, Bolun and Towers, Mark and {Perez-Vicente}, Rodrigo and Willems, Lucas and Lahlou, Salem and Pal, Suman and Castro, Pablo Samuel and Terry, Jordan},
  year = 2023,
  journal = {Advances in Neural Information Processing Systems},
  volume = {36},
  pages = {73383--73394},
  urldate = {2025-05-05}
}

@article{cox_probability_1946,
  title = {Probability, Frequency and Reasonable Expectation},
  author = {Cox, Richard T.},
  year = 1946,
  journal = {American journal of physics},
  volume = {14},
  number = {1},
  pages = {1--13},
  publisher = {American Association of Physics Teachers},
  urldate = {2025-05-06}
}

@article{cutler_dynamic_1979,
  title = {Dynamic {{Matrix Control-A Computer Control Algorithm}}},
  author = {Cutler, Richard R. and Ramaker, B. L.},
  year = 1979,
  journal = {Proc. Joint Automatic Control Conference, 1979},
  urldate = {2025-05-01}
}

@article{dacosta_active_2020,
  title = {Active Inference on Discrete State-Spaces: {{A}} Synthesis},
  shorttitle = {Active Inference on Discrete State-Spaces},
  author = {Da Costa, Lancelot and Parr, Thomas and Sajid, Noor and Veselic, Sebastijan and Neacsu, Victorita and Friston, Karl},
  year = 2020,
  month = dec,
  journal = {Journal of Mathematical Psychology},
  volume = {99},
  pages = {102447},
  issn = {0022-2496},
  doi = {10.1016/j.jmp.2020.102447},
  urldate = {2023-08-07},
  langid = {english}
}

@misc{dacosta_active_2024,
  title = {Active {{Inference}} as a {{Model}} of {{Agency}}},
  author = {Da Costa, Lancelot and Tenka, Samuel and Zhao, Dominic and Sajid, Noor},
  year = 2024,
  month = jan,
  number = {arXiv:2401.12917},
  eprint = {2401.12917},
  primaryclass = {cs},
  publisher = {arXiv},
  doi = {10.48550/arXiv.2401.12917},
  urldate = {2025-05-04},
  archiveprefix = {arXiv}
}

@inproceedings{dauwels_variational_2007,
  title = {On {{Variational Message Passing}} on {{Factor Graphs}}},
  booktitle = {{{IEEE International Symposium}} on {{Information Theory}}},
  author = {Dauwels, J.},
  year = 2007,
  month = jun,
  pages = {2546--2550},
  address = {Nice, France},
  doi = {10.1109/ISIT.2007.4557602}
}

@misc{devries_expected_2025,
  title = {Expected {{Free Energy-based Planning}} as {{Variational Inference}}},
  author = {De Vries, Bert and Nuijten, Wouter and {van de Laar}, Thijs and Kouw, Wouter and Adamiat, Sepideh and Nisslbeck, Tim and Lukashchuk, Mykola and Nguyen, Hoang Minh Huu and Araya, Marco Hidalgo and Tresor, Raphael and Jenneskens, Thijs and Nikoloska, Ivana and Subramanian, Raaja Ganapathy and van Erp, Bart and Bagaev, Dmitry and Podusenko, Albert},
  year = 2025,
  month = apr,
  number = {arXiv:2504.14898},
  eprint = {2504.14898},
  primaryclass = {stat},
  publisher = {arXiv},
  doi = {10.48550/arXiv.2504.14898},
  urldate = {2025-05-01},
  archiveprefix = {arXiv}
}

@article{forney_codes_2001,
  title = {Codes on Graphs: {{Normal}} Realizations},
  shorttitle = {Codes on Graphs},
  author = {Forney, G. David},
  year = 2001,
  journal = {IEEE Transactions on Information Theory},
  volume = {47},
  number = {2},
  pages = {520--548},
  publisher = {IEEE},
  urldate = {2025-05-05}
}

@article{friston_active_2015,
  title = {Active Inference and Epistemic Value},
  author = {Friston, Karl and Rigoli, Francesco and Ognibene, Dimitri and Mathys, Christoph and Fitzgerald, Thomas and Pezzulo, Giovanni},
  year = 2015,
  month = oct,
  journal = {Cognitive Neuroscience},
  volume = {6},
  number = {4},
  pages = {187--214},
  issn = {1758-8928, 1758-8936},
  doi = {10.1080/17588928.2015.1020053},
  urldate = {2025-05-04},
  langid = {english}
}

@article{friston_active_2025,
  title = {Active {{Inference}} and {{Intentional Behavior}}},
  author = {Friston, Karl J. and Salvatori, Tommaso and Isomura, Takuya and Tschantz, Alexander and Kiefer, Alex and Verbelen, Tim and Koudahl, Magnus and Paul, Aswin and Parr, Thomas and Razi, Adeel and Kagan, Brett J. and Buckley, Christopher L. and Ramstead, Maxwell J. D.},
  year = 2025,
  month = mar,
  journal = {Neural Computation},
  volume = {37},
  number = {4},
  pages = {666--700},
  issn = {0899-7667},
  doi = {10.1162/neco_a_01738},
  urldate = {2025-05-06}
}

@article{friston_freeenergy_2010,
  title = {The Free-Energy Principle: A Unified Brain Theory?},
  shorttitle = {The Free-Energy Principle},
  author = {Friston, Karl},
  year = 2010,
  month = feb,
  journal = {Nature Reviews Neuroscience},
  volume = {11},
  number = {2},
  pages = {127--138},
  publisher = {Nature Publishing Group},
  issn = {1471-0048},
  doi = {10.1038/nrn2787},
  urldate = {2023-08-15},
  copyright = {2010 Springer Nature Limited},
  langid = {english}
}

@article{friston_sophisticated_2021,
  title = {Sophisticated {{Inference}}},
  author = {Friston, Karl and Da Costa, Lancelot and Hafner, Danijar and Hesp, Casper and Parr, Thomas},
  year = 2021,
  month = mar,
  journal = {Neural Computation},
  volume = {33},
  number = {3},
  pages = {713--763},
  issn = {0899-7667},
  doi = {10.1162/neco_a_01351},
  urldate = {2021-12-22}
}

@article{ito_kullback_2022,
  title = {Kullback--{{Leibler}} Control for Discrete-Time Nonlinear Systems on Continuous Spaces},
  author = {Ito, Kaito and Kashima, Kenji},
  year = 2022,
  month = jun,
  journal = {SICE Journal of Control, Measurement, and System Integration},
  volume = {15},
  number = {2},
  pages = {119--129},
  publisher = {Taylor \& Francis},
  issn = {1882-4889},
  doi = {10.1080/18824889.2022.2095827},
  urldate = {2025-05-01}
}

@book{jaynes_probability_2003,
  title = {Probability {{Theory}}: {{The Logic}} of {{Science}}},
  shorttitle = {Probability {{Theory}}},
  author = {Jaynes, E. T.},
  year = 2003,
  month = apr,
  edition = {1},
  publisher = {Cambridge University Press},
  doi = {10.1017/CBO9780511790423},
  urldate = {2023-07-25},
  isbn = {978-0-521-59271-0 978-0-511-79042-3}
}

@article{kappen_optimal_2012,
  title = {Optimal Control as a Graphical Model Inference Problem},
  author = {Kappen, B. and Gomez, V. and Opper, M.},
  year = 2012,
  month = may,
  journal = {Machine Learning},
  volume = {87},
  number = {2},
  eprint = {0901.0633},
  primaryclass = {cs, math},
  pages = {159--182},
  issn = {0885-6125, 1573-0565},
  doi = {10.1007/s10994-012-5278-7},
  urldate = {2024-07-31},
  archiveprefix = {arXiv}
}

@misc{kingma_autoencoding_2022,
  title = {Auto-{{Encoding Variational Bayes}}},
  author = {Kingma, Diederik P. and Welling, Max},
  year = 2022,
  month = dec,
  number = {arXiv:1312.6114},
  eprint = {1312.6114},
  primaryclass = {cs, stat},
  publisher = {arXiv},
  doi = {10.48550/arXiv.1312.6114},
  urldate = {2024-05-27},
  archiveprefix = {arXiv}
}

@book{koller_probabilistic_2009,
  title = {Probabilistic {{Graphical Models}}: {{Principles}} and {{Techniques}}},
  shorttitle = {Probabilistic {{Graphical Models}}},
  author = {Koller, Daphne and Friedman, Nir},
  year = 2009,
  month = jul,
  publisher = {MIT Press},
  googlebooks = {7dzpHCHzNQ4C},
  isbn = {978-0-262-01319-2},
  langid = {english}
}

@inproceedings{lazaro-gredilla_what_2024,
  title = {What Type of Inference Is Planning?},
  booktitle = {Advances in {{Neural Information Processing Systems}}},
  author = {{L{\'a}zaro-Gredilla}, Miguel and Ku, Li Yang and Murphy, Kevin P. and George, Dileep},
  editor = {Globerson, A. and Mackey, L. and Belgrave, D. and Fan, A. and Paquet, U. and Tomczak, J. and Zhang, C.},
  year = 2024,
  volume = {37},
  pages = {116705--116742},
  publisher = {Curran Associates, Inc.},
  doi = {10.52202/079017-3705}
}

@misc{levine_reinforcement_2018,
  title = {Reinforcement {{Learning}} and {{Control}} as {{Probabilistic Inference}}: {{Tutorial}} and {{Review}}},
  shorttitle = {Reinforcement {{Learning}} and {{Control}} as {{Probabilistic Inference}}},
  author = {Levine, Sergey},
  year = 2018,
  month = may,
  number = {arXiv:1805.00909},
  eprint = {1805.00909},
  primaryclass = {cs},
  publisher = {arXiv},
  doi = {10.48550/arXiv.1805.00909},
  urldate = {2025-04-24},
  archiveprefix = {arXiv}
}

@article{loeliger_factor_2007,
  title = {The {{Factor Graph Approach}} to {{Model-Based Signal Processing}}},
  author = {Loeliger, Hans-Andrea and Dauwels, Justin and Hu, Junli and Korl, Sascha and Ping, Li and Kschischang, Frank R.},
  year = 2007,
  month = jun,
  journal = {Proceedings of the IEEE},
  volume = {95},
  number = {6},
  pages = {1295--1322},
  issn = {0018-9219},
  doi = {10.1109/JPROC.2007.896497},
  urldate = {2014-04-10}
}

@article{loeliger_introduction_2004,
  title = {An Introduction to Factor Graphs},
  author = {Loeliger, H.-A.},
  year = 2004,
  month = jan,
  journal = {IEEE Signal Processing Magazine},
  volume = {21},
  number = {1},
  pages = {28--41},
  issn = {1558-0792},
  doi = {10.1109/MSP.2004.1267047}
}

@inproceedings{paul_active_2021,
  title = {Active {{Inference}} for {{Stochastic Control}}},
  booktitle = {Machine {{Learning}} and {{Principles}} and {{Practice}} of {{Knowledge Discovery}} in {{Databases}}},
  author = {Paul, Aswin and Sajid, Noor and Gopalkrishnan, Manoj and Razi, Adeel},
  editor = {Kamp, Michael and Koprinska, Irena and Bibal, Adrien and Bouadi, Tassadit and Fr{\'e}nay, Beno{\^i}t and Gal{\'a}rraga, Luis and Oramas, Jos{\'e} and Adilova, Linara and Krishnamurthy, Yamuna and Kang, Bo and Largeron, Christine and Lijffijt, Jefrey and Viard, Tiphaine and Welke, Pascal and Ruocco, Massimiliano and Aune, Erlend and Gallicchio, Claudio and Schiele, Gregor and Pernkopf, Franz and Blott, Michaela and Fr{\"o}ning, Holger and Schindler, G{\"u}nther and Guidotti, Riccardo and Monreale, Anna and Rinzivillo, Salvatore and Biecek, Przemyslaw and Ntoutsi, Eirini and Pechenizkiy, Mykola and Rosenhahn, Bodo and Buckley, Christopher and Cialfi, Daniela and Lanillos, Pablo and Ramstead, Maxwell and Verbelen, Tim and Ferreira, Pedro M. and Andresini, Giuseppina and Malerba, Donato and Medeiros, Ib{\'e}ria and {Fournier-Viger}, Philippe and Nawaz, M. Saqib and Ventura, Sebastian and Sun, Meng and Zhou, Min and Bitetta, Valerio and Bordino, Ilaria and Ferretti, Andrea and Gullo, Francesco and Ponti, Giovanni and Severini, Lorenzo and Ribeiro, Rita and Gama, Jo{\~a}o and Gavald{\`a}, Ricard and Cooper, Lee and Ghazaleh, Naghmeh and Richiardi, Jonas and Roqueiro, Damian and Saldana Miranda, Diego and Sechidis, Konstantinos and Gra{\c c}a, Guilherme},
  year = 2021,
  volume = {1524},
  pages = {669--680},
  publisher = {Springer International Publishing},
  address = {Cham},
  doi = {10.1007/978-3-030-93736-2_47},
  urldate = {2025-05-06},
  isbn = {978-3-030-93735-5 978-3-030-93736-2},
  langid = {english}
}

@article{paul_efficient_2024,
  title = {On Efficient Computation in Active Inference},
  author = {Paul, Aswin and Sajid, Noor and Costa, Lancelot Da and Razi, Adeel},
  year = 2024,
  month = nov,
  journal = {Expert Systems with Applications},
  volume = {253},
  eprint = {2307.00504},
  primaryclass = {cs},
  pages = {124315},
  issn = {0957-4174},
  doi = {10.1016/j.eswa.2024.124315},
  urldate = {2025-05-01},
  archiveprefix = {arXiv}
}

@inproceedings{pearl_reverend_1982,
  title = {Reverend {{Bayes}} on {{Inference Engines}}: {{A Distributed Hierarchical Approach}}},
  shorttitle = {Reverend {{Bayes}} on {{Inference Engines}}},
  booktitle = {{{AAAI-82 Proceedings}}},
  author = {Pearl, Judea},
  year = 1982,
  pages = {133--136},
  publisher = {AAAI Press},
  address = {Carnegie Mellon University, Pittsburgh PA},
  urldate = {2025-05-06}
}

@book{pontryagin_mathematical_2018,
  title = {Mathematical {{Theory}} of {{Optimal Processes}}},
  author = {Pontryagin, L. S.},
  year = 2018,
  month = may,
  publisher = {Routledge},
  address = {London},
  doi = {10.1201/9780203749319},
  isbn = {978-0-203-74931-9}
}

@article{richalet_algorithmic_1976,
  title = {Algorithmic Control of Industrial Processes},
  author = {Richalet, J.},
  year = 1976,
  journal = {Proc. of the 4\textasciicircum th IFAC Sympo. on Identification and System Parameter Estimation},
  pages = {1119--1167},
  urldate = {2025-05-01}
}

@article{richalet_model_1978,
  title = {Model Predictive Heuristic Control: {{Applications}} to Industrial Processes},
  shorttitle = {Model Predictive Heuristic Control},
  author = {Richalet, J. and Rault, A. and Testud, J. L. and Papon, J.},
  year = 1978,
  month = sep,
  journal = {Automatica},
  volume = {14},
  number = {5},
  pages = {413--428},
  issn = {0005-1098},
  doi = {10.1016/0005-1098(78)90001-8},
  urldate = {2025-05-01}
}

@phdthesis{senoz_message_2022,
  type = {Phd {{Thesis}} 1 ({{Research TU}}/e / {{Graduation TU}}/e)},
  title = {Message {{Passing Algorithms}} for {{Hierarchical Dynamical Models}}},
  author = {{\textcommabelow S}en{\"o}z, {\.I}smail},
  year = 2022,
  month = jun,
  address = {Eindhoven},
  isbn = {978-90-386-5532-1},
  school = {Eindhoven University of Technology}
}

@inproceedings{todorov_general_2008,
  title = {General Duality between Optimal Control and Estimation},
  booktitle = {2008 47th {{IEEE Conference}} on {{Decision}} and {{Control}}},
  author = {Todorov, Emanuel},
  year = 2008,
  month = dec,
  pages = {4286--4292},
  issn = {0191-2216},
  doi = {10.1109/CDC.2008.4739438},
  urldate = {2025-05-04}
}

@inproceedings{todorov_linearlysolvable_2006,
  title = {Linearly-Solvable {{Markov}} Decision Problems},
  booktitle = {Advances in {{Neural Information Processing Systems}}},
  author = {Todorov, Emanuel},
  year = 2006,
  volume = {19},
  publisher = {MIT Press},
  urldate = {2025-05-01}
}

@article{winn_variational_2005,
  title = {Variational {{Message Passing}}.},
  author = {Winn, John and Bishop, Christopher},
  year = 2005,
  month = apr,
  journal = {Journal of Machine Learning Research},
  volume = {6},
  pages = {661--694}
}

@article{yedidia_constructing_2005,
  title = {Constructing Free-Energy Approximations and Generalized Belief Propagation Algorithms},
  author = {Yedidia, Jonathan S. and Freeman, W.T. and Weiss, Y.},
  year = 2005,
  month = jul,
  journal = {IEEE Transactions on Information Theory},
  volume = {51},
  number = {7},
  pages = {2282--2312},
  issn = {0018-9448},
  doi = {10.1109/TIT.2005.850085}
}
\appendix
\section{Proof of \autoref{cor:bethe}} \label{sec:bethe-proof}
\begin{proof} Proof of \autoref{cor:bethe}. This proof is an adjusted proof of the proof of \autoref{the:efetheorem}, which is given in \cite{devries_expected_2025}.
    \begin{subequations}
        \begin{align}
            F[q] & = E_{q(\bm{y}, \bm{x}, \bm{u})}\bigg[ \log \frac{q(\bm{y}, \bm{x}, \bm{u})}{p(\bm{y}, \bm{x}, \bm{u})  \hat{p}(\bm{x}) \tilde{p}(\bm{u}) \tilde{p}(\bm{x})  } \bigg]                                               \\
                 & = E_{q(\bm{u})}\bigg[ \log \frac{q(\bm{u})}{p(\bm{u})}
                + \underbrace{E_{q(\bm{y}, \bm{x} | \bm{u})}\big[ \log \frac{q(\bm{y}, \bm{x} | \bm{u})}{p(\bm{y}, \bm{x}|\bm{u})  \hat{p}(\bm{x}) \tilde{p}(\bm{u}) \tilde{p}(\bm{x})  }\big]}_{C(\bm{u})}
            \bigg] \label{eq:F-C-1}                                                                                                                                                                                                   \\
                 & = E_{q(\bm{u})}\bigg[ \log \frac{q(\bm{u})}{p(\bm{u})}
            + \underbrace{G(\bm{u}) +E_{q(\bm{y}, \bm{x} | \bm{u})} \big[\log \frac{q(\bm{y}, \bm{x}|\bm{u})}{p(\bm{y}, \bm{x}|\bm{u})}\big] + constant}_{C(\bm{u}) \text{ if \eqref{eq:epistemic-priors-bethe} holds}}
            \bigg] \label{eq:F-C-2}                                                                                                                                                                                                   \\
                 & = E_{q(\bm{u})}\big[ G(\bm{u})\big]+ E_{q(\bm{y}, \bm{x}, \bm{u})}\bigg[\log \frac{q(\bm{y}, \bm{x}, \bm{u})}{p(\bm{y}, \bm{x}, \bm{u})}\bigg] + constant \quad \text{if \eqref{eq:epistemic-priors-bethe} holds }
        \end{align}
    \end{subequations}
\end{proof}

In the above derivation, we still need to prove the transition for $C(\bm{u}) $ from
\eqref{eq:F-C-1} to \eqref{eq:F-C-2}, which we address next.

\begin{lemma}[Proof of equivalence $C(\bm{u})$ in \eqref{eq:F-C-1} and \eqref{eq:F-C-2}]
    \begin{subequations}\label{eq:proof-Cu}
        \begin{align}
            C( & \bm{u}) = \mathbb{E}_{q(\bm{y}, \bm{x}|\bm{u})}\bigg[ \log \frac{ \overbrace{q(\bm{y}, \bm{x}|\bm{u})}^{\text{posterior}} }{ \underbrace{p(\bm{y}, \bm{x}|\bm{u})}_{\text{predictive}} \underbrace{\hat{p}(\bm{x})}_{\text{utility}} \underbrace{\tilde{p}(\bm{u}) \tilde{p}(\bm{x}) }_{\text{epistemic priors}}} \bigg] \label{eq:Cu-first-line}                                     \\
               & = \underbrace{ \mathbb{E}_{q(\bm{y}, \bm{x}|\bm{u})}\bigg[\log\bigg( \underbrace{\frac{q(\bm{x}|\bm{u})}{\hat{p}(\bm{x})}}_{\text{risk}}\cdot \underbrace{\frac{1}{q(\bm{y}|\bm{x})}}_{\text{ambiguity}} \bigg) \bigg] }_{G(\bm{u}) = \text{Expected Free Energy}} +   \label{eq:Cu-eqC}                                                                                              \\
               & \quad + \mathbb{E}_{q(\bm{y}, \bm{x}|\bm{u})}\bigg[ \log\bigg( \underbrace{\frac{\hat{p}(\bm{x}) q(\bm{y}|\bm{x})}{q(\bm{x}|\bm{u})}}_{\text{inverse factors from }G(\bm{u})} \cdot \underbrace{\frac{q(\bm{y}, \bm{x}|\bm{u})}{p(\bm{y}, \bm{x}|\bm{u}) \hat{p}(\bm{x}) \tilde{p}(\bm{u}) \tilde{p}(\bm{x})    }}_{\text{factors from }\eqref{eq:Cu-first-line}} \bigg)\bigg] \notag \\
               & = G(\bm{u}) + \underbrace{\mathbb{E}_{q(\bm{y}, \bm{x}|\bm{u})}\bigg[ \log \frac{q(\bm{y}, \bm{x}|\bm{u})}{p(\bm{y}, \bm{x}|\bm{u})}\bigg]}_{=B(\bm{u})} + \underbrace{\mathbb{E}_{q(\bm{y}, \bm{x}|\bm{u})}\bigg[ \log  \frac{q(\bm{y}|\bm{x})}{q(\bm{x}|\bm{u}) \tilde{p}(\bm{u}) \tilde{p}(\bm{x}) } \bigg]}_{\text{choose epistemic priors to let this vanish}}                   \\
               & = G(\bm{u}) + B(\bm{u}) + \mathbb{E}_{q(\bm{x}|\bm{u})}\bigg[\log \frac{1}{q(\bm{x}|\bm{u}) \tilde{p}(\bm{u})} \bigg] + \mathbb{E}_{q(\bm{y} ,\bm{x} |\bm{u})}\bigg[ \log  \frac{q(\bm{y}|\bm{x})}{\tilde{p}(\bm{x})} \bigg] \notag
        \end{align}
    \end{subequations}
    Now here we can replace the general $q(\bm{y}|\bm{x})$ and $ q(\bm{x}|\bm{u})$ with the factorised $\prod_t q(y_t|x_t)$ and $\prod_t q(x_t|x_{t-1}, u_t)$.
    \begin{subequations}
        \begin{align}
            C( & \bm{u}) = G(\bm{u}) + B(\bm{u}) + \mathbb{E}_{q(\bm{x}|\bm{u})}\bigg[\log \frac{1}{\prod_t q(x_t|x_{t-1}, u_t) \tilde{p}(u_t)} \bigg] + \notag                                \\
               & \quad + \mathbb{E}_{q(\bm{x}|\bm{u})}\bigg[\mathbb{E}_{q(\bm{y}|\bm{x})}\bigg[ \log q(y_t|x_t) - \log \tilde{p}(x_t) \bigg] \bigg]                                            \\
               & = G(\bm{u}) + B(\bm{u})  \notag                                                                                                                                               \\
               & \quad + \sum_{t=1}^{T} \mathbb{E}_{q(x_t, x_{t-1} | u_t)}\bigg[- \log q(x_t|x_{t-1}, u_t)  - \log \tilde{p}(u_t) \bigg] + \notag                                              \\
               & \quad + \sum_{t=1}^{T} \mathbb{E}_{q(x_t, | u_t)}\bigg[ \mathbb{E}_{q(y_t|x_t)}\bigg[ \log   q(y_t|x_t) - \log \tilde{p}(x_t) \bigg] \bigg]\,. \label{eq:Cu-before-entropies}
        \end{align}
    \end{subequations}
    Now we can recognize the following:

    \begin{subequations}
        \begin{align}
             & \mathbb{E}_{q(x_t, x_{t-1} | u_t)}\bigg[- \log q(x_t|x_{t-1}, u_t) \bigg]                                                \\
             & \qquad = \mathbb{E}_{q(x_t, x_{t-1} | u_t)}\bigg[- \bigg(\log q(x_t, x_{t-1} | u_t) - \log q(x_{t-1} | u_t)\bigg) \bigg] \\
             & \qquad = H[q(x_t, x_{t-1} | u_t)] - H[q(x_{t-1} | u_t)] \,,
        \end{align}
    \end{subequations}
    and
    \begin{align}
        \mathbb{E}_{q(y_t|x_t)}\big[ \log   q(y_t|x_t) \big] = -H[q(y_t|x_t)] \,.
    \end{align}
    Which, when substituted into \eqref{eq:Cu-before-entropies}, together with the definitions of $\tilde{p}(u_t)$ and $\tilde{p}(x_t)$, yields
    \begin{subequations}
        \begin{align}
             & = G(\bm{u}) + B(\bm{u})  \notag                                                                                                                                                                                                   \\
             & \quad +  \sum_{t=1}^{T}  \underbrace{H[q(x_t, x_{t-1} | u_t)] - H[q(x_{t-1} | u_t)]  - \log \tilde{p}(u_t)}_{=c_x \text{ if } \tilde{p}(u_t) \propto \exp(H[q(x_t, x_{t-1} | u_t)] - H[q(x_{t-1} | u_t)])}  + \notag              \\
             & \quad +  \sum_{t=1}^{T} \mathbb{E}_{q(x_t, | u_t)}\bigg[ \underbrace{-H[q(y_t | x_t)] - \log \tilde{p}(x_t)}_{=c_y \text{ if } \tilde{p}(x_t) \propto \exp(-H[q(y_t | x_t)])} \bigg]   \label{eq:Cu-after-entropies-substitution} \\
             & = G(\bm{u}) + \mathbb{E}_{q(\bm{y}, \bm{x}|\bm{u})}\bigg[ \log \frac{q(\bm{y}, \bm{x}|\bm{u})}{p(\bm{y}, \bm{x}|\bm{u})}\bigg] + c_x + c_y \,, \quad \text{if \eqref{eq:epistemic-priors-bethe} holds.}
        \end{align}
    \end{subequations}
\end{lemma}

\section{Generative Model for the Gridworld Environment} \label{appx:gridworld-generative-model}

The generative model for the stochastic grid environment is defined as follows:
\begin{subequations}
    \begin{align}
    x_0 & \sim p(x_0)                              \\
    x_t & \sim \text{Cat}(x_t | x_{t-1}, u_{t}, B) \\
    y_t & \sim \text{Cat}(y_t | x_t, A)            \\
    x_T & \sim \hat{p}(x_T) \,.
\end{align}
\end{subequations}

Here, $s_t$ represents the agent's state at time $t$, $y_t$ is the observation, and $u_t$ is the action. The transition dynamics are governed by $B$, and $A$ represents the observation model. The agent starts with a prior belief $p(s_0)$ and aims to reach the goal state by the end of the planning horizon $T$.

In the case of the KL-control agent, the prior on the control is given by
\begin{equation}
    u_t \sim \text{Cat}(u_t | \mathbf{1}/4) \quad \text{for } t = 1,\ldots,T \,.
\end{equation}
The EFE-minimizing agent uses empirical priors on the control and the states, given by

\begin{subequations}
\begin{align}
    u_t & \sim \text{Cat}(u_t | \sigma(H[q(x_t, x_{t-1} | u_t)] - H[q(x_{t-1} | u_t)])) \quad &\text{for } t = 1,\ldots,T \\
    x_t & \sim \text{Cat}(x_t | \sigma(-H[q(y_t | x_t)])) \quad &\text{for } t = 1,\ldots,T \,.
\end{align}
\end{subequations}

\section{Generative Model for the Minigrid Environment} \label{appx:minigrid-generative-model}

The generative model for the Minigrid environment uses a factorized state and observation space, which makes the model computationally tractable.
Here, the location of the agent is denoted by $l$, the orientation by $o$, the key-door state by $s$, the door location by $d$, and the key location by $k$.
The key-door state is a categorical variable with three possible values: $\{0, 1, 2\}$, where $0$ indicates that the key is not picked up yet, $1$ indicates that the key is picked up but the door is not opened yet, and $2$ indicates that the key is picked up and the door is opened.
For the observations, $y_{t,(x,y)}$ is the observation at time $t$ for cell $(x,y)$ of the field of view.
The generative model for the Minigrid environment is defined as follows:
\begin{subequations}
\begin{align}
    l_0         & \sim p(l_0)                                                                                            \\
    o_0         & \sim p(o_0)                                                                                            \\
    s_0         & \sim p(s_0)                                                                                            \\
    d           & \sim p(d)                                                                                              \\
    k           & \sim p(k)                                                                                              \\
    l_t         & \sim \text{Cat}(l_t | l_{t-1}, o_{t-1}, k, d, s_{t-1}, u_{t}, B^l)                                     \\
    o_t         & \sim \text{Cat}(o_t | o_{t-1}, B^o, u_t)                                                               \\
    s_t         & \sim \text{Cat}(s_t | s_{t-1}, l_{t-1}, o_{t-1}, k, d, u_{t-1}, B^s)                                   \\
    y_{t,(x,y)} & \sim \text{Cat}(y_{t,(x,y)} | l_t, o_t, k, d, s_t, A_{(x,y)}) \quad \forall (x,y) \in \{1,\ldots,7\}^2 
\end{align}
\end{subequations}

with terminal state goal priors

\begin{subequations}

\begin{align}
    l_T & \sim \hat{p}(l_T)                               \\
    s_T & \sim \text{Cat}(s_T | [0, 0, 1]) = \hat{p}(s_T) \,,
\end{align}
\end{subequations}
where $B^l$ is the location transition tensor, $B^o$ is the orientation transition tensor, $B^s$ is the key-door state transition tensor, and $A_{(x,y)}$ are the observation tensors for each cell in the field of view.

In the case of the KL-control agent, the prior on the control is given by
\begin{equation}
    u_t \sim \text{Cat}(u_t | \mathbf{1}/5) \quad \text{for } t = 1,\ldots,T \,.
\end{equation}
The EFE-minimizing agent uses empirical priors on the control and the states, given by
\begin{subequations}
\begin{align}
    u_t & \sim \text{Cat}(u_t | \sigma(                                                                                                                &        \\
        & \qquad H[q(l_t, l_{t-1}, o_{t-1}, k, d, s_{t-1} | u_t)] - H[q(l_{t-1}, o_{t-1}, k, d, s_{t-1} | u_t)]                                     +  &        \\
        & \qquad H[q(s_t, s_{t-1}, l_t, o_{t-1}, k, d | u_t)] - H[q(s_{t-1}, l_t, o_{t-1}, k, d | u_t)]                                             +  &        \\
        & \qquad H[q(o_t, o_{t-1} | u_t)] - H[q(o_{t-1} | u_t)])) \quad                                          \text{for } t = 1,\ldots,T            &        \\
    l_t & \sim \text{Cat}\bigg(l_t | \sigma \bigg(\sum_{(x,y)}  -H[q(y_{t,(x,y)}, o_t, s_t, k, d, | l_t)] + H[o_t, k, d, s_t | l_t] \bigg)\bigg)       & \notag \\
        & \quad \text{for } t = 1,\ldots,T                                                                                                             &        \\
    o_t & \sim \text{Cat}\bigg(o_t | \sigma \bigg(\sum_{(x,y)}  -H[q(y_{t,(x,y)}, l_t, s_t, k, d, | o_t)] + H[l_t, s_t, k, d | o_t] \bigg)\bigg)       & \notag \\
        & \quad \text{for } t = 1,\ldots,T                                                                                                             &        \\
    s_t & \sim \text{Cat}\bigg(s_t | \sigma \bigg(\sum_{(x,y)}  -H[q(y_{t,(x,y)}, l_t, o_t, k, d, | s_t)] + H[l_t, o_t, k, d | s_t] \bigg)\bigg)       & \notag \\
        & \quad \text{for } t = 1,\ldots,T                                                                                                             &        \\
    k   & \sim \text{Cat}\bigg(k | \sigma  \bigg(\sum_t \sum_{(x,y) } -H[q(y_{t,(x,y)}, l_t, o_t, s_t, d, | k)] + H[l_t, o_t, s_t, d | k] \bigg)\bigg) & \\
    d   & \sim \text{Cat}\bigg(d | \sigma  \bigg(\sum_t \sum_{(x,y) } -H[q(y_{t,(x,y)}, l_t, o_t, s_t, k, | d)] + H[l_t, o_t, s_t, k | d] \bigg)\bigg)\,. &  
\end{align}

\end{subequations}
\section{Additional Results for the Stochastic Grid Environment Experiments} \label{appx:additional_results_gridworld}
In this section, we will provide further analysis of the results presented in \autoref{sec:stochastic-grid-environment-results}. We will go into more detail on the convergence of the Bethe Free Energy over different iterations of the variational inference procedure, and we will elaborate on a trajectory in a specific episode.
\subsection{Convergence Analysis}
In \autoref{fig:gridworld_inference_results}, we plot the Bethe Free Energy over the iterations of the message passing procedure, along the state of the environment at which the inference procedure is being called.

As we can see, even though we have not provided a proof of convergence, in this specific example, the Bethe Free Energy converges to a constant value, indicating that our inference procedure has converged.

Note that the Bethe Free Energy is an approximation of the true Variational Free Energy, and can therefore not be used for model comparison \cite{yedidia_constructing_2005}. Although RxInfer minimizes the Bethe Free Energy, this explains the upwards trend in the Bethe Free Energy curve, and we can only use the Bethe Free Energy as a sanity check to check convergence of the inference procedure.
\begin{figure}[htbp]
    \centering
    \resizebox{\textwidth}{!}{\input{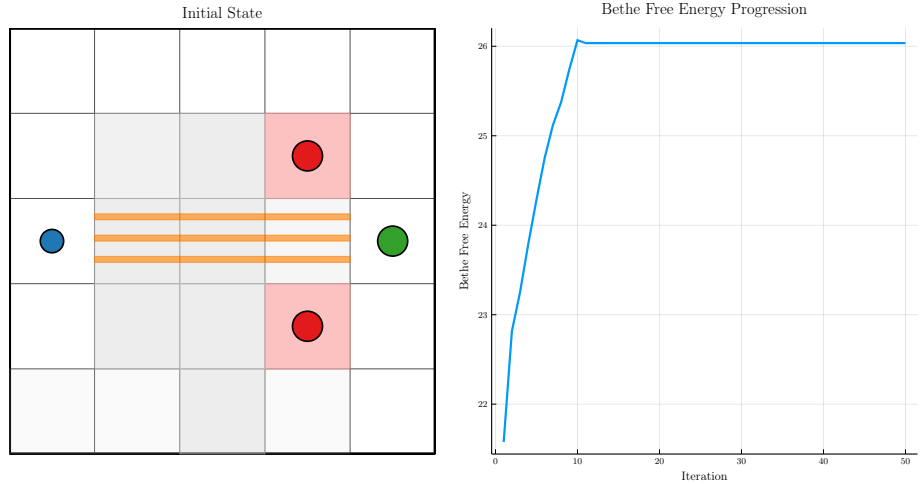}}
    \caption{Visualization of the inference results for the stochastic grid environment. On the left, the initial state of the environment is shown. On the right we show the Bethe Free Energy curve over the iterations of message passing. Convergence to a constant value indicates convergence of the inference procedure.}
    \label{fig:gridworld_inference_results}
\end{figure}

\subsection{Trajectory}

Figure \ref{fig:maze-comparison-full} provides a frame-by-frame comparison of the trajectories taken by the EFE-minimizing agent (left) and the KL-control agent (right) in the stochastic grid environment. This visualization clearly demonstrates the differences in planning strategies between the two approaches.
\begin{figure*}[tbh]
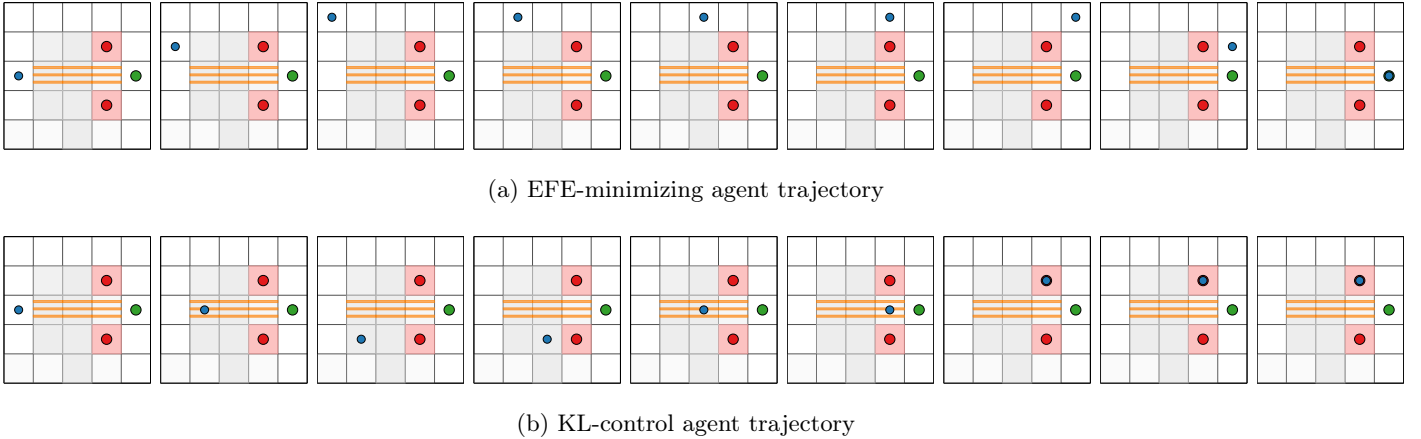

    \centering
    \begin{adjustwidth}{-3cm}{-3cm} % Extend 2cm into left and right margins
        % EFE-minimizing agent film strip
        \begin{subfigure}{\linewidth}
            \centering
            \setlength{\tabcolsep}{1pt} % Reduce spacing between images
            \begin{tabular}{*{9}{c}}
                \resizebox{0.11\linewidth}{!}{\input{figures/stochastic_maze/efe_stochastic_maze_agent/frames/frame_000.tikz}} &
                \resizebox{0.11\linewidth}{!}{\input{figures/stochastic_maze/efe_stochastic_maze_agent/frames/frame_001.tikz}} &
                \resizebox{0.11\linewidth}{!}{\input{figures/stochastic_maze/efe_stochastic_maze_agent/frames/frame_002.tikz}} &
                \resizebox{0.11\linewidth}{!}{\input{figures/stochastic_maze/efe_stochastic_maze_agent/frames/frame_003.tikz}} &
                \resizebox{0.11\linewidth}{!}{\input{figures/stochastic_maze/efe_stochastic_maze_agent/frames/frame_004.tikz}} &
                \resizebox{0.11\linewidth}{!}{\input{figures/stochastic_maze/efe_stochastic_maze_agent/frames/frame_005.tikz}} &
                \resizebox{0.11\linewidth}{!}{\input{figures/stochastic_maze/efe_stochastic_maze_agent/frames/frame_006.tikz}} &
                \resizebox{0.11\linewidth}{!}{\input{figures/stochastic_maze/efe_stochastic_maze_agent/frames/frame_007.tikz}} &
                \resizebox{0.11\linewidth}{!}{\input{figures/stochastic_maze/efe_stochastic_maze_agent/frames/frame_008.tikz}}   \\
            \end{tabular}
            \caption{EFE-minimizing agent trajectory}
        \end{subfigure}

        \vspace{1em}

        % KL-control agent film strip
        \begin{subfigure}{\linewidth}
            \centering
            \setlength{\tabcolsep}{1pt} % Reduce spacing between images
            \begin{tabular}{*{9}{c}}
                \resizebox{0.11\linewidth}{!}{\input{figures/stochastic_maze/klcontrol_stochastic_maze_agent/frames/frame_000.tikz}} &
                \resizebox{0.11\linewidth}{!}{\input{figures/stochastic_maze/klcontrol_stochastic_maze_agent/frames/frame_001.tikz}} &
                \resizebox{0.11\linewidth}{!}{\input{figures/stochastic_maze/klcontrol_stochastic_maze_agent/frames/frame_002.tikz}} &
                \resizebox{0.11\linewidth}{!}{\input{figures/stochastic_maze/klcontrol_stochastic_maze_agent/frames/frame_003.tikz}} &
                \resizebox{0.11\linewidth}{!}{\input{figures/stochastic_maze/klcontrol_stochastic_maze_agent/frames/frame_004.tikz}} &
                \resizebox{0.11\linewidth}{!}{\input{figures/stochastic_maze/klcontrol_stochastic_maze_agent/frames/frame_005.tikz}} &
                \resizebox{0.11\linewidth}{!}{\input{figures/stochastic_maze/klcontrol_stochastic_maze_agent/frames/frame_006.tikz}} &
                \resizebox{0.11\linewidth}{!}{\input{figures/stochastic_maze/klcontrol_stochastic_maze_agent/frames/frame_007.tikz}} &
                \resizebox{0.11\linewidth}{!}{\input{figures/stochastic_maze/klcontrol_stochastic_maze_agent/frames/frame_008.tikz}}   \\
            \end{tabular}
            \caption{KL-control agent trajectory}
        \end{subfigure}
    \end{adjustwidth}

    \caption{Comparison of agent trajectories in a stochastic maze environment. Top: EFE-minimizing agent with epistemic priors. Bottom: KL-control agent without epistemic priors.}
    \label{fig:maze-comparison-full}
\end{figure*}

% Supervisor's comment addressed with implementation above
% \tvdl{In \autoref{fig:maze-comparison-full} you could arrange the images in a horizontal film strip. Then the time subscripts are not needed. Or just indicate the movement with arrows, since the environment does not change.}

The EFE-minimizing agent immediately chooses the longer but safer path, moving upward and around the cells with stochastic transitions. This risk-averse behavior is a direct result of the epistemic priors that penalize uncertainty in transitions. By frame $t=8$, the agent has successfully reached the goal state without encountering any hazardous transitions.

In contrast, the KL-control agent attempts to optimize for the shortest path, moving directly through cells with stochastic transitions. This optimistic planning is characteristic of approaches that don't account for aleatoric uncertainty. While this strategy would be optimal in a deterministic environment, it leads to potential failures in this stochastic setting because the agent cannot manipulate its own luck.

The difference in trajectories directly translates to the performance gap observed across the 100 trial episodes. The EFE-minimizing agent's perfect success rate (100\%) compared to the KL-control agent's lower performance (21\%) confirms the theoretical prediction that incorporating epistemic uncertainty leads to more robust planning in stochastic environments.

\section{Additional Results for the Minigrid Environment Experiments} \label{appx:additional_results_minigrid}

\subsection{Convergence Analysis}
In \autoref{fig:minigrid_inference_results}, we perform inference on the initial state of the Minigrid environment shown in \autoref{fig:appx-minigrid_initial_state}. The figure displays the Bethe Free Energy progression during the inference process, along with the agent's final beliefs about its current location, orientation, and the state of the key and door after the last iteration. We observe that the BFE stabilizes to a constant value, indicating that our inference procedure successfully converges.
\begin{figure}[tbh]
    \centering
    \includegraphics[width=0.25\textwidth]{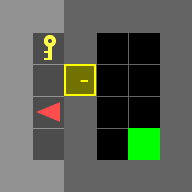}
    \caption{Initial state of the Minigrid environment.}
    \label{fig:appx-minigrid_initial_state}
\end{figure}
\begin{figure}[bht]
    \centering
    \resizebox{\textwidth}{!}{\input{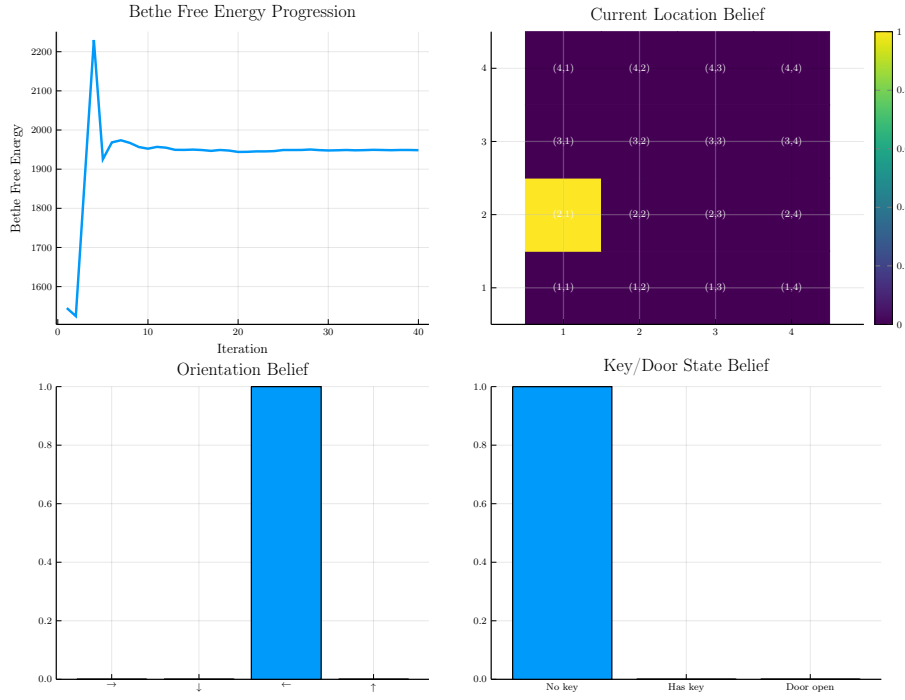}}
    \caption{Visualization of the inference results for the Minigrid environment. Top left: Bethe Free Energy curve over the iterations of message passing. Top right: Agent's belief of its current location after the last iteration. Bottom left: Agent's belief of its current orientation after the last iteration. Bottom right: Agent's belief of the state of the key and door after the last iteration.}
    \label{fig:minigrid_inference_results}
\end{figure}

\subsection{Trajectory}
Figures \ref{fig:efe_trajectory_grid} and \ref{fig:klcontrol_trajectory_grid} provide a frame-by-frame comparison of the trajectories taken by the EFE-minimizing agent and the KL-control agent in the Minigrid environment. This visualization clearly demonstrates the differences in planning strategies between the two approaches, and highlights the shortcomings of the KL-control approach.

The EFE-minimizing agent is able to solve the task at hand, while the KL-control agent stays in the corner of the grid facing the wall. As shown in Figure \ref{fig:efe_trajectory_grid}, the EFE-minimizing agent reaches the goal state, while the KL-control agent does not.

The difference in trajectories directly translates to the performance gap observed across the test episodes. The EFE-minimizing agent's superior performance confirms our theoretical prediction that incorporating epistemic uncertainty leads to more efficient planning in partially observable environments like Minigrid.

\begin{figure}[tbh]
    \centering
    \begin{tabular}{ccccc}
        \includegraphics[width=0.18\textwidth]{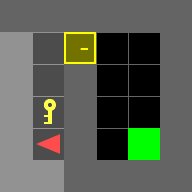} &
        \includegraphics[width=0.18\textwidth]{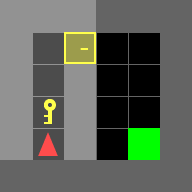} &
        \includegraphics[width=0.18\textwidth]{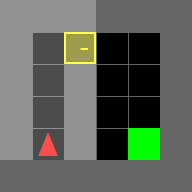} &
        \includegraphics[width=0.18\textwidth]{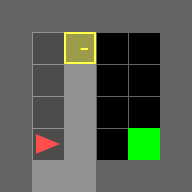} &
        \includegraphics[width=0.18\textwidth]{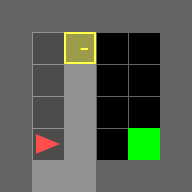}                                                                     \\
        \small{$t=0$}                                                              & \small{$t=1$}  & \small{$t=2$}  & \small{$t=3$}  & \small{$t=4$}  \\

        \includegraphics[width=0.18\textwidth]{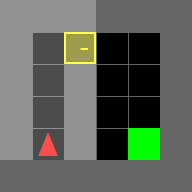} &
        \includegraphics[width=0.18\textwidth]{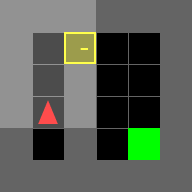} &
        \includegraphics[width=0.18\textwidth]{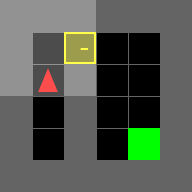} &
        \includegraphics[width=0.18\textwidth]{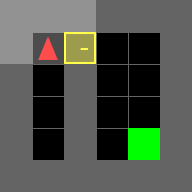} &
        \includegraphics[width=0.18\textwidth]{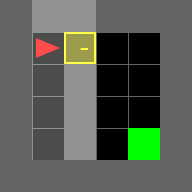}                                                                     \\
        \small{$t=5$}                                                              & \small{$t=6$}  & \small{$t=7$}  & \small{$t=8$}  & \small{$t=9$}  \\

        \includegraphics[width=0.18\textwidth]{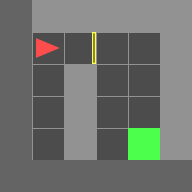} &
        \includegraphics[width=0.18\textwidth]{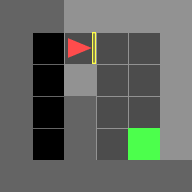} &
        \includegraphics[width=0.18\textwidth]{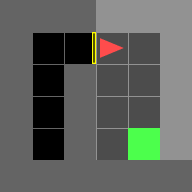} &
        \includegraphics[width=0.18\textwidth]{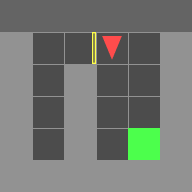} &
        \includegraphics[width=0.18\textwidth]{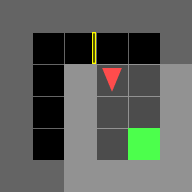}                                                                     \\
        \small{$t=10$}                                                             & \small{$t=11$} & \small{$t=12$} & \small{$t=13$} & \small{$t=14$} \\

        \includegraphics[width=0.18\textwidth]{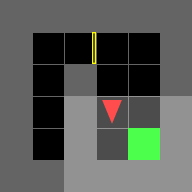} &
        \includegraphics[width=0.18\textwidth]{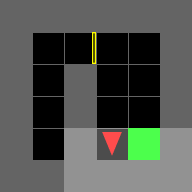} &
        \includegraphics[width=0.18\textwidth]{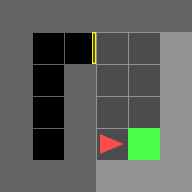} &
        \includegraphics[width=0.18\textwidth]{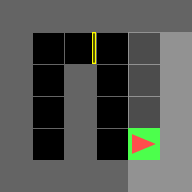} &
        \includegraphics[width=0.18\textwidth]{figures/minigrid/efe/frame_018.png}                                                                     \\
        \small{$t=15$}                                                             & \small{$t=16$} & \small{$t=17$} & \small{$t=18$} & \small{$t=19$} \\

        \includegraphics[width=0.18\textwidth]{figures/minigrid/efe/frame_018.png} &
        \includegraphics[width=0.18\textwidth]{figures/minigrid/efe/frame_018.png} &
        \includegraphics[width=0.18\textwidth]{figures/minigrid/efe/frame_018.png} &
        \includegraphics[width=0.18\textwidth]{figures/minigrid/efe/frame_018.png} &
        \includegraphics[width=0.18\textwidth]{figures/minigrid/efe/frame_018.png}                                                                     \\
        \small{$t=20$}                                                             & \small{$t=21$} & \small{$t=22$} & \small{$t=23$} & \small{$t=24$} \\
    \end{tabular}
    \caption{Visualization of the agent's trajectory in the Minigrid environment using EFE-based control. The 5$\times$5 grid shows the sequential frames of the agent's movement.}
    \label{fig:efe_trajectory_grid}
\end{figure}

\begin{figure}[tbh]
    \centering
    \begin{tabular}{ccccc}
        \includegraphics[width=0.18\textwidth]{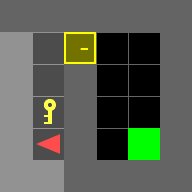} &
        \includegraphics[width=0.18\textwidth]{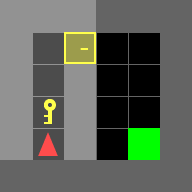} &
        \includegraphics[width=0.18\textwidth]{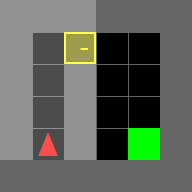} &
        \includegraphics[width=0.18\textwidth]{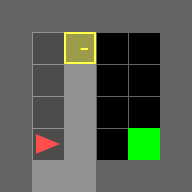} &
        \includegraphics[width=0.18\textwidth]{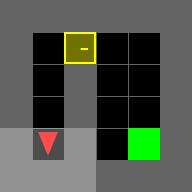}                                                                     \\
        \small{$t=0$}                                                                    & \small{$t=1$}  & \small{$t=2$}  & \small{$t=3$}  & \small{$t=4$}  \\

        \includegraphics[width=0.18\textwidth]{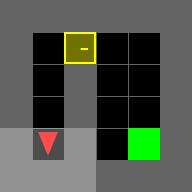} &
        \includegraphics[width=0.18\textwidth]{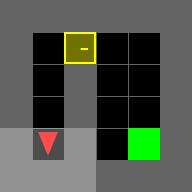} &
        \includegraphics[width=0.18\textwidth]{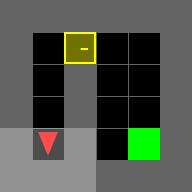} &
        \includegraphics[width=0.18\textwidth]{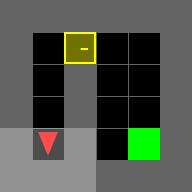} &
        \includegraphics[width=0.18\textwidth]{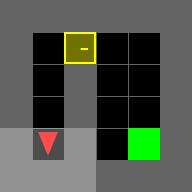}                                                                     \\
        \small{$t=5$}                                                                    & \small{$t=6$}  & \small{$t=7$}  & \small{$t=8$}  & \small{$t=9$}  \\

        \includegraphics[width=0.18\textwidth]{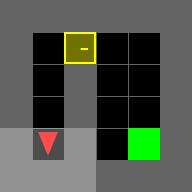} &
        \includegraphics[width=0.18\textwidth]{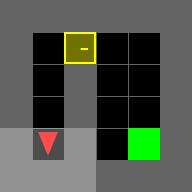} &
        \includegraphics[width=0.18\textwidth]{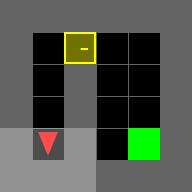} &
        \includegraphics[width=0.18\textwidth]{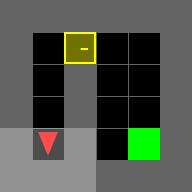} &
        \includegraphics[width=0.18\textwidth]{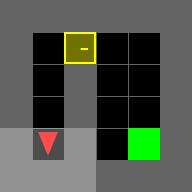}                                                                     \\
        \small{$t=10$}                                                                   & \small{$t=11$} & \small{$t=12$} & \small{$t=13$} & \small{$t=14$} \\

        \includegraphics[width=0.18\textwidth]{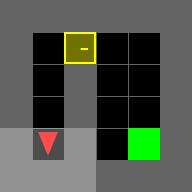} &
        \includegraphics[width=0.18\textwidth]{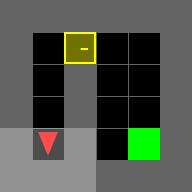} &
        \includegraphics[width=0.18\textwidth]{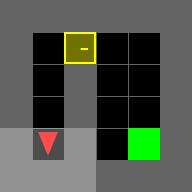} &
        \includegraphics[width=0.18\textwidth]{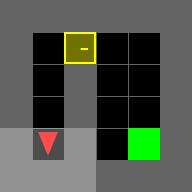} &
        \includegraphics[width=0.18\textwidth]{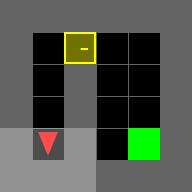}                                                                     \\
        \small{$t=15$}                                                                   & \small{$t=16$} & \small{$t=17$} & \small{$t=18$} & \small{$t=19$} \\

        \includegraphics[width=0.18\textwidth]{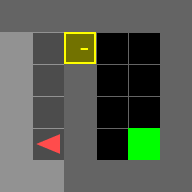} &
        \includegraphics[width=0.18\textwidth]{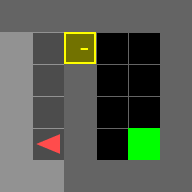} &
        \includegraphics[width=0.18\textwidth]{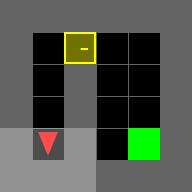} &
        \includegraphics[width=0.18\textwidth]{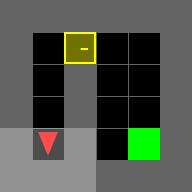} &
        \includegraphics[width=0.18\textwidth]{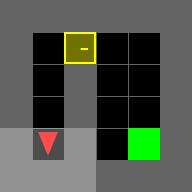}                                                                     \\
        \small{$t=20$}                                                                   & \small{$t=21$} & \small{$t=22$} & \small{$t=23$} & \small{$t=24$} \\
    \end{tabular}
    \caption{Visualization of the agent's trajectory in the Minigrid environment using KL control. The 5$\times$5 grid shows the sequential frames of the agent's movement throughout the episode.}
    \label{fig:klcontrol_trajectory_grid}
\end{figure}

\end{document}